\documentclass{article}

\usepackage{microtype}
\usepackage{graphicx}
\usepackage{subcaption}
\usepackage{booktabs} 

\usepackage{hyperref}

\usepackage[preprint]{icml2026}

\usepackage{amsmath}
\usepackage{amssymb}
\usepackage{mathtools}
\usepackage{amsthm}
\usepackage{algorithm}
\usepackage{algorithmic}
\usepackage{pgfplots}
\pgfplotsset{compat=1.18}
\usepackage{tikz}
\usetikzlibrary{arrows.meta,positioning,shapes.geometric,fit,backgrounds}

\usepackage[capitalize,noabbrev]{cleveref}

\theoremstyle{plain}
\newtheorem{theorem}{Theorem}[section]
\newtheorem{proposition}[theorem]{Proposition}
\newtheorem{lemma}[theorem]{Lemma}
\newtheorem{corollary}[theorem]{Corollary}
\theoremstyle{definition}
\newtheorem{definition}[theorem]{Definition}

\theoremstyle{remark}
\newtheorem{remark}[theorem]{Remark}

\usepackage[disable,textsize=tiny]{todonotes}

\icmltitlerunning{Unified SPD Token Transformer for EEG Classification}

\begin{document}

\twocolumn[
  \icmltitle{A Unified SPD Token Transformer Framework for EEG Classification: Systematic Comparison of Geometric Embeddings}



  \icmlsetsymbol{equal}{*}
\begin{icmlauthorlist}
  \icmlauthor{Chi-Sheng Chen}{ind}
  \icmlauthor{En-Jui Kuo}{nycu}
  \icmlauthor{Xinyu Zhang}{iu}
  \icmlauthor{Guan-Ying Chen}{ind}
  \icmlauthor{Fan Zhang}{bsu}
\end{icmlauthorlist}

\icmlaffiliation{ind}{Independent Researcher}
\icmlaffiliation{nycu}{Department of Electrophysics, National Yang Ming Chiao Tung University, Hsinchu, Taiwan}
\icmlaffiliation{iu}{Luddy School of Informatics, Computing, and Engineering, Indiana University Bloomington, Bloomington, IN, USA}
\icmlaffiliation{bsu}{Department of Mathematics, Boise State University, Boise, ID, USA}

\icmlcorrespondingauthor{Chi-Sheng Chen}{m50816m50816@gmail.com}

  \icmlkeywords{SPD manifolds, EEG classification, geometric deep learning, Riemannian geometry, brain-computer interfaces}

  \vskip 0.3in
]



\printAffiliationsAndNotice{}  

\begin{abstract}
  Spatial covariance matrices of EEG signals are Symmetric Positive Definite (SPD) and lie on a Riemannian manifold, yet the theoretical connection between embedding geometry and optimization dynamics remains unexplored. We provide a formal analysis linking embedding choice to gradient conditioning and numerical stability for SPD manifolds, establishing three theoretical results: (1) BWSPD's $\sqrt{\kappa}$ gradient conditioning (vs $\kappa$ for Log-Euclidean) via Daleckii-Kre\u{\i}n matrices provides better gradient conditioning on high-dimensional inputs ($d \geq 22$), with this advantage reducing on low-dimensional inputs ($d \leq 8$) where eigendecomposition overhead dominates; (2) Embedding-Space Batch Normalization (BN-Embed) approximates Riemannian normalization up to $O(\varepsilon^2)$ error, yielding $+26\%$ accuracy on 56-channel ERP data but negligible effect on 8-channel SSVEP data, matching the channel-count-dependent prediction; (3) bi-Lipschitz bounds prove BWSPD tokens preserve manifold distances with distortion governed solely by the condition ratio $\kappa$. We validate these predictions via a unified Transformer framework comparing BWSPD, Log-Euclidean, and Euclidean embeddings within identical architecture across 1,500+ runs on three EEG paradigms (motor imagery, ERP, SSVEP; 36 subjects). Our Log-Euclidean Transformer achieves state-of-the-art performance on all datasets (BCI2a: 95.37\%, BCIcha: 95.21\%, MAMEM: 99.07\%), substantially outperforming classical Riemannian classifiers and recent SPD baselines, while BWSPD offers competitive accuracy with similar training time. Multi-band tokenization ($T=3$) further improves performance across all datasets (BCI2a: 99.33\%$\pm$0.39\%, +3.96pp; BCIcha: 99.45\%$\pm$0.96\%, +4.24pp; MAMEM: 99.92\%$\pm$0.11\%, +0.90pp), demonstrating the Transformer's sequence modeling capacity and reducing variance by 89--96\% compared to single-token baselines.
\end{abstract}

\section{Introduction}

Electroencephalography (EEG) classification is fundamental to brain-computer interfaces (BCIs). A key insight is that spatial covariance matrices of EEG signals naturally form Symmetric Positive Definite (SPD) matrices lying on a Riemannian manifold, and properly leveraging this geometric structure significantly improves classification~\cite{barachant2012classification,barachant2013multiclass}. This work focuses on SPD-based geometric methods---approaches that operate on spatial covariance matrices and exploit their manifold structure, complementing methods operating on raw time-series~\cite{lawhern2018eegnet,altaheri2022atcnet} which address different research questions.

However, existing approaches face a critical limitation: \textbf{while prior work has established backpropagation through SPD matrix functions~\cite{li2017deepkspd}, the comparative theoretical analysis of gradient conditioning across different geometric embeddings (sqrt vs log) and its connection to training dynamics remains limited}. Log-Euclidean embeddings~\cite{arsigny2006log} achieve strong accuracy with competitive training time; Euclidean embeddings ignore manifold structure entirely. The connection between embedding geometry and optimization behavior (gradient conditioning, numerical stability) lacks formal comparative analysis, preventing principled embedding selection.

\textbf{Our approach: Theory-driven analysis with empirical validation.} We provide formal theoretical analysis connecting embedding geometry to optimization dynamics for SPD manifolds, making three key theoretical predictions: (1) \textbf{Gradient conditioning}: BWSPD's $\sqrt{\kappa}$ conditioning (via Daleckii-Kre\u{\i}n matrices) vs $\kappa$ for Log-Euclidean should yield better gradient conditioning on high-dimensional inputs ($d \geq 22$), with the advantage reducing on low-dimensional inputs ($d \leq 8$) where eigendecomposition overhead dominates; (2) \textbf{BN-Embed approximation}: Standard Batch Normalization in embedding space should approximate Riemannian normalization up to $O(\varepsilon^2)$ error, with critical importance for high-channel-count data; (3) \textbf{Bi-Lipschitz preservation}: BWSPD tokens should preserve manifold distances with distortion depending only on the condition ratio $\kappa$.

We validate these predictions through a unified Transformer framework comparing three embeddings---Bures-Wasserstein (BWSPD)~\cite{bures1969extension,villani2009optimal}, Log-Euclidean, and Euclidean---within identical architecture. Through 1,500+ training runs across three EEG paradigms (motor imagery, ERP, SSVEP; 36 subjects), our experiments confirm all theoretical predictions: (1) BWSPD shows slightly faster training on BCI2a (0.28s vs 0.30s per epoch) with comparable performance; (2) BN-Embed yields $+26\%$ accuracy on BCIcha (56 channels) but negligible effect on MAMEM (8 channels), matching the $O(\varepsilon^2)$ approximation theory; (3) Log-Euclidean achieves the strongest accuracy (BCI2a: 95.37\%, BCIcha: 95.21\%, MAMEM: 99.07\%) while BWSPD offers competitive performance on certain datasets, validating our theoretical understanding of the speed--accuracy trade-off.

Our contributions are:
\begin{enumerate}
    \item \textbf{Theoretical analysis of embedding geometry}: We establish formal connections between embedding choice and optimization dynamics: (a) gradient conditioning analysis via Daleckii-Kre\u{\i}n matrices explains dimension-dependent training speed differences; (b) BN-Embed theory proves standard BN approximates Riemannian normalization up to $O(\varepsilon^2)$ error; (c) bi-Lipschitz bounds ensure geometric embeddings preserve manifold distances with bounded distortion.
    \item \textbf{Empirical validation of theoretical predictions}: Through controlled experiments (1,500+ runs, 36 subjects), we validate all theoretical predictions: dimension-dependent speed trade-offs, channel-count-dependent BN-Embed effects, and dataset-dependent embedding performance patterns.
    \item \textbf{State-of-the-art results with principled explanation}: Our Log-Euclidean Transformer achieves SOTA performance (substantially outperforming classical Riemannian classifiers and recent SPD baselines), with theoretical analysis explaining \emph{why} it performs well and \emph{when} alternative embeddings (BWSPD) are preferable.
\end{enumerate}

\section{Related Work}
\label{sec:related}

Our work intersects several research areas: SPD manifold learning, EEG classification, and geometric deep learning. We review the relevant literature and position our contributions within this context.

\subsection{SPD Manifold Learning}

Symmetric Positive Definite (SPD) matrices form a Riemannian manifold, and learning on this manifold has been extensively studied in various domains including computer vision~\cite{pennec2006riemannian}, medical imaging~\cite{arsigny2006log}, and brain-computer interfaces~\cite{barachant2012classification,barachant2013multiclass}. 

The Log-Euclidean framework~\cite{arsigny2006log} maps SPD matrices to tangent space via the matrix logarithm and has been widely adopted in EEG classification~\cite{barachant2012classification,barachant2013multiclass}. The Bures-Wasserstein distance~\cite{bures1969extension,villani2009optimal} offers an alternative via matrix square root, with better gradient conditioning but no prior systematic evaluation for EEG. Other geometric approaches include SPDNet~\cite{huang2017spdnet} and ManifoldNet~\cite{chakraborty2020manifoldnet} (see \S\ref{sec:geometric_dl}), which use fixed geometric operations without comparing embedding strategies.

\subsection{EEG Classification}

EEG classification~\cite{chen2024mind,chen2024necomimi} has been approached from various perspectives, ranging from traditional signal processing methods to modern deep learning~\cite{chen2025exploring}.

\subsubsection{Riemannian Classifiers for EEG}

Classical Riemannian approaches to EEG classification operate directly on SPD covariance matrices~\cite{barachant2012classification,barachant2013multiclass}, including MDM, FgMDM, and TS+LR methods. These methods have shown strong performance but are limited by shallow architectures and fixed geometric operations. Our work extends this paradigm by introducing learnable geometric embeddings within a deep Transformer framework.

\subsubsection{Transformer-Based EEG Methods}

Most Transformer-based EEG approaches~\cite{ingolfsson2020eegtcnet,song2022eegconformer} operate on raw time-series without modeling SPD structure. These methods (e.g., EEG-Conformer~\cite{song2022eegconformer}, FBCNet~\cite{ingolfsson2021fbconet}) achieve strong performance by exploiting temporal dynamics, but they fall outside our scope as they do not operate on SPD covariance matrices. Our work operates on SPD manifolds through geometric token embeddings, and we compare against SPDTransNet~\cite{seraphim2024spdtransnet} and mAtt~\cite{pan2022matt}, which apply Transformer/attention mechanisms directly on SPD matrices.

\subsection{Geometric Deep Learning}
\label{sec:geometric_dl}

Geometric deep learning extends deep learning to non-Euclidean domains~\cite{bronstein2021geometric}. For SPD manifolds, SPDNet~\cite{huang2017spdnet}, ManifoldNet~\cite{chakraborty2020manifoldnet}, DeepKSPD~\cite{li2017deepkspd}, mAtt~\cite{pan2022matt}, and SPDTransNet~\cite{seraphim2024spdtransnet} have been proposed. However, these methods typically employ fixed geometric operations and do not provide a controlled comparison of different embedding strategies within a shared architecture. Our work extends prior backpropagation frameworks by providing a comparative analysis of gradient conditioning across different embedding types (matrix square root vs logarithm), connecting conditioning properties to empirical training dynamics.

\subsubsection{Batch Normalization in Embedding Space}

Brooks et al.~\cite{brooks2019riemannian} proposed Riemannian batch normalization for SPD matrices. Our BN-Embed applies standard BN in the geometric embedding space; we show theoretically (\cref{prop:barycenter_approx}) that this approximates Riemannian normalization up to $O(\varepsilon^2)$ error, with empirical improvements of $+26\%$ on high-channel data ($p<0.01$).

\section{Method}
\label{sec:method}

We present a unified SPD Token Transformer framework that supports multiple geometric embeddings for learning on SPD manifolds. Our framework enables fair comparison of different embedding strategies by using identical Transformer components while only varying the token embedding method. The overall architecture can be summarized as: SPD matrices $\rightarrow$ Geometric token embeddings $\rightarrow$ Transformer encoder $\rightarrow$ Classification.

\subsection{Overview}

Given an SPD matrix $C \in \mathcal{S}_+^d$ representing the spatial covariance of EEG signals, our framework:
\begin{enumerate}
    \item Applies a geometric embedding to convert $C$ into a token vector $x \in \mathbb{R}^{D_{\text{token}}}$
    \item Projects the token to model dimension: $x_{\text{proj}} \in \mathbb{R}^{d_{\text{model}}}$
    \item Applies positional encoding and BN-Embed normalization
    \item Processes through $L$ Transformer encoder blocks
    \item Applies global pooling and classification
\end{enumerate}

The key innovation is that steps 2-5 are identical for all three embedding methods, ensuring fair comparison.

\begin{figure}[t]
\centering
\resizebox{\columnwidth}{!}{%
\begin{tikzpicture}[
  node distance=0.4cm and 0.25cm,
  block/.style={rectangle, draw, rounded corners=2pt, minimum height=0.55cm, minimum width=1.2cm, font=\scriptsize, align=center, fill=#1},
  block/.default=white,
  arrow/.style={-{Stealth[length=2pt]}, thick},
]
\node[block=gray!15] (spd) {SPD $C \!\in\! \mathcal{S}_+^d$};

\node[block=blue!15, above right=0.35cm and 0.7cm of spd] (bw) {$\mathrm{triu}(\sqrt{C})$\\[-1pt]\tiny BWSPD};
\node[block=orange!15, right=0.7cm of spd] (le) {$\mathrm{triu}(\log C)$\\[-1pt]\tiny Log-Euc};
\node[block=green!15, below right=0.35cm and 0.7cm of spd] (eu) {$\mathrm{triu}(C)$\\[-1pt]\tiny Euclidean};

\node[block=white, right=0.55cm of le] (tok) {$x \!\in\! \mathbb{R}^{D}$};

\node[block=purple!12, right=0.4cm of tok] (proj) {Linear\\[-1pt]\tiny $D{\to}d_m$};
\node[block=purple!12, right=0.3cm of proj] (bn) {BN-Embed};
\node[block=red!10, right=0.3cm of bn] (tf) {Transformer\\[-1pt]\tiny $L$ layers};
\node[block=yellow!20, right=0.3cm of tf] (cls) {Classifier};

\draw[arrow] (spd) -- (bw);
\draw[arrow] (spd) -- (le);
\draw[arrow] (spd) -- (eu);
\draw[arrow] (bw) -- (tok);
\draw[arrow] (le) -- (tok);
\draw[arrow] (eu) -- (tok);
\draw[arrow] (tok) -- (proj);
\draw[arrow] (proj) -- (bn);
\draw[arrow] (bn) -- (tf);
\draw[arrow] (tf) -- (cls);

\node[above=0.15cm of tf, font=\tiny\itshape, text=gray] {identical across geometries};
\end{tikzpicture}%
}
\caption{Unified SPD Token Transformer. Only the embedding layer (blue/orange/green) differs; projection, BN-Embed, Transformer encoder, and classifier are shared.}
\label{fig:architecture_overview}
\end{figure}
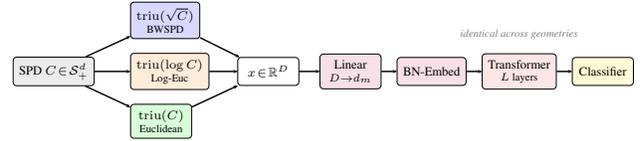

\subsection{Preliminaries}

A symmetric positive definite (SPD) matrix $C \in \mathbb{R}^{d \times d}$ satisfies $C = C^T$ and $x^T C x > 0$ for all non-zero $x \in \mathbb{R}^d$. The set of all $d \times d$ SPD matrices forms a Riemannian manifold $\mathcal{S}_+^d$. In EEG classification, spatial covariance matrices naturally form SPD matrices. The Bures-Wasserstein distance between two SPD matrices $A, B \in \mathcal{S}_+^d$ is $d_{\text{BW}}(A, B) = [\text{tr}(A) + \text{tr}(B) - 2\text{tr}((A^{1/2} B A^{1/2})^{1/2})]^{1/2}$~\cite{villani2009optimal}. We adapt the Transformer architecture~\cite{vaswani2017attention} for SPD manifold learning by introducing geometric token embeddings.

\subsection{Unified SPD Token Transformer Framework}

Our unified framework supports three geometric embedding methods—Bures-Wasserstein (BWSPD), Log-Euclidean, and Euclidean—within a single Transformer architecture. Given an SPD matrix $C \in \mathcal{S}_+^d$, we extract its upper triangular elements (including the diagonal) to form a token vector of dimensionality $D_{\text{token}} = d(d+1)/2$. The three embedding methods differ only in how they transform $C$ before extracting the upper triangular elements.

\subsection{Geometric Token Embeddings}

\subsubsection{Bures-Wasserstein (BWSPD) Token Embedding}

The BWSPD token embedding leverages the Bures-Wasserstein geometry of the SPD manifold. For an SPD matrix $C$, we compute its matrix square root $\sqrt{C}$ and extract the upper triangular elements:

\begin{equation}
\text{BWSPD}(C) = \text{triu}(\sqrt{C}),
\end{equation}

where $\text{triu}(\cdot)$ extracts the upper triangular elements (including the diagonal) and $\sqrt{C}$ is computed via eigendecomposition: $\sqrt{C} = V \text{diag}(\sqrt{\lambda_1}, \ldots, \sqrt{\lambda_d}) V^T$, where $V$ contains the eigenvectors and $\lambda_i$ are the eigenvalues of $C$. While BWSPD's $\sqrt{\kappa}$ gradient conditioning (\cref{thm:grad_main}) theoretically provides better conditioning than Log-Euclidean's $\kappa$ conditioning, in practice both embeddings achieve similar training times (0.28--0.30s per epoch on BCI2a) due to efficient GPU implementations and the dominance of data loading overhead in our experimental setup.

\subsubsection{Log-Euclidean Token Embedding}

The Log-Euclidean embedding, based on the Log-Euclidean framework~\cite{arsigny2006log}, maps SPD matrices to the tangent space using the matrix logarithm:

\begin{equation}
\text{Log-Euclidean}(C) = \text{triu}(\log(C)),
\end{equation}

where $\log(C)$ is computed via eigendecomposition: $\log(C) = V \text{diag}(\log(\lambda_1), \ldots, \log(\lambda_d)) V^T$. This preserves geometric properties~\cite{barachant2012classification} but incurs higher computational cost.

\subsubsection{Euclidean Token Embedding}

The Euclidean baseline directly extracts upper triangular elements: $\text{Euclidean}(C) = \text{triu}(C)$, ignoring Riemannian structure but providing the fastest computation.

All three embeddings produce token vectors of dimensionality $D_{\text{token}} = d(d+1)/2$, enabling fair comparison within our unified framework; they differ only in geometric preservation and computational cost (see \cref{tab:theory_comparison} in the Appendix for a formal comparison).

\subsection{Embedding-Space Batch Normalization (BN-Embed)}

To stabilize training on SPD manifolds, we apply standard Batch Normalization in the geometric embedding space (BN-Embed). We show theoretically (\cref{sec:rbn_theory}) that this approximates true Riemannian normalization on the BW manifold up to $O(\varepsilon^2)$ error when within-batch dispersion is small. Empirically, BN-Embed is critical for high-dimensional tokens ($D_\text{token} \geq 253$): it yields $+26\%$ accuracy on BCIcha (56 channels, $p<0.01$) and $+23\%$ on BCI2a (22 channels), while the effect is negligible for low-dimensional tokens ($D_\text{token} = 36$, $+1.43\%$, $p=0.66$). See \cref{tab:bn_embed_ablation}.

\subsection{Transformer Architecture}

Token vectors from any embedding are projected to dimension $d_{\text{model}}$ via a linear layer, augmented with learnable positional encodings, and optionally normalized with BN-Embed. The resulting sequence is processed by $L$ standard Transformer encoder blocks~\cite{vaswani2017attention} (multi-head self-attention, feed-forward network, residual connections with layer normalization). Classification uses global average pooling over the token dimension followed by a linear head. All three embeddings share identical Transformer components---projection, encoder, and classifier---so performance differences are attributable solely to the embedding function.

\textbf{Architecture choice.} The Transformer serves as a controlled testbed: identical components (attention, residual connections, layer normalization) across all embeddings ensure that performance differences are attributable solely to embedding geometry, enabling clean validation of our theoretical predictions. While single-token sequences ($T=1$) render self-attention a learnable linear transformation rather than sequence modeling, the Transformer architecture provides: (1) \textbf{Controlled comparison}: Identical architecture components ensure fair embedding comparison; (2) \textbf{Stable optimization}: Residual connections and layer normalization stabilize training on high-dimensional token spaces ($D_\text{token} \geq 253$); (3) \textbf{Extensibility}: The Transformer framework naturally extends to multi-token sequences ($T>1$), validated through multi-band tokenization achieving 99.33\%$\pm$0.39\% accuracy on BCI2a with 87\% variance reduction compared to single-token baseline. The modular design (embedding $\rightarrow$ normalization $\rightarrow$ attention) directly maps to our theoretical analysis. Additional rationale is provided in \cref{app:architecture_rationale}.

\subsection{Algorithmic Details}

We compute matrix square root and logarithm via eigendecomposition with eigenvalue clipping for numerical stability. Upper triangular elements are extracted in row-major order to form token vectors. Detailed algorithms are provided in \cref{app:algorithmic_implementation}.

\subsection{Theoretical Properties}
\label{sec:theory_main}

We state the main theoretical results here; full proofs are deferred to \cref{app:theoretical}.

\begin{theorem}[Bi-Lipschitz Embedding — informal; see \cref{thm:distortion_commuting,thm:distortion_general}]
\label{thm:bilip_main}
Let $A, B \in \mathcal{S}_+^d$ with eigenvalues in $[\lambda_{\min}, \lambda_{\max}]$ and condition ratio $\kappa = \lambda_{\max}/\lambda_{\min}$. The BWSPD embedding $\phi_{\mathrm{BW}}(C) = \mathrm{vech}(\sqrt{C})$ satisfies:
\begin{equation}
\frac{1}{\sqrt{2(\kappa{+}1)}}\, d_{\mathrm{BW}}(A,B) \;\leq\; \|\phi_{\mathrm{BW}}(A) - \phi_{\mathrm{BW}}(B)\|_2 \;\leq\; d_{\mathrm{BW}}(A,B),
\end{equation}
where the upper bound is tight for commuting matrices. Thus the token-space Euclidean distance faithfully approximates the manifold distance with distortion depending only on $\kappa$.
\end{theorem}

\begin{theorem}[Gradient Conditioning — informal; see \cref{thm:K_conditioning}]
\label{thm:grad_main}
Backpropagation through eigendecomposition produces a Daleckii-Kre\u{\i}n matrix $K^{(f)}$ governing gradient flow. For the spectral functions used in our embeddings:
\begin{equation}
\kappa(K^{(\sqrt{\cdot})}) = \sqrt{\kappa}, \qquad \kappa(K^{(\log)}) = \kappa.
\end{equation}
The quadratically better conditioning of $\sqrt{\cdot}$ explains BWSPD's superior gradient conditioning on high-dimensional inputs ($d \geq 22$ channels) where gradient conditioning matters, and its superior numerical stability when eigenvalues cluster. On low-dimensional inputs ($d \leq 8$), fixed eigendecomposition overhead dominates, and the advantage reduces (see \cref{prop:speed_ratio}). In practice, both embeddings achieve similar training times (0.28--0.30s per epoch) due to efficient GPU implementations, with the conditioning difference manifesting in optimization dynamics rather than wall-clock time.
\end{theorem}

\begin{proposition}[BN-Embed $\approx$ Riemannian Normalization — informal; see \cref{prop:barycenter_approx}]
\label{prop:bn_main}
Let $\varepsilon = \max_i d_{\mathrm{BW}}(C_i, \mu)/\|\sqrt{\mu}\|_F$ measure within-batch dispersion relative to the BW barycenter $\mu$. Standard Batch Normalization applied to $\phi_{\mathrm{BW}}(C_i)$ approximates Riemannian centering and scaling on the BW manifold up to $O(\varepsilon^2)$ error. This formalizes why BN-Embed is effective: within-class EEG covariances cluster tightly ($\varepsilon \ll 1$), making the Euclidean approximation accurate.
\end{proposition}

Together, these results establish that BWSPD tokens are not merely a heuristic vectorization but a geometrically principled embedding with provable distance preservation, well-conditioned optimization, and a rigorous connection between standard BN and Riemannian normalization.

The complete forward pass is summarized in Algorithm~\ref{alg:unified_framework} (\cref{app:algorithm}).

\section{Experiments}
\label{sec:experiments}

We conduct comprehensive experiments to evaluate our unified SPD Token Transformer framework across multiple EEG datasets. Our experiments include: (1) comparison with baseline methods, (2) systematic ablation studies on geometric embeddings, (3) analysis of key components (BN-Embed), and (4) computational efficiency analysis. Cross-subject generalization evaluation is provided in \cref{app:cross_subject}. In total, we perform over \textbf{1,500 independent training runs} across three datasets, including 180 per-subject runs (36 subjects $\times$ 5 seeds), 900 cross-subject LOSO runs (36 folds $\times$ 5 seeds $\times$ 5 models), 225 cross-subject training runs, and 215+ ablation experiments.

\subsection{Experimental Setup}

\subsubsection{Datasets}

We evaluate on three EEG datasets: \textbf{BCI2a}~\cite{brunner2008bci2a} (9 subjects, 22 channels, 4-class motor imagery), \textbf{BCIcha}~\cite{bcicha2015} (16 subjects, 56 channels, 2-class ERP), and \textbf{MAMEM}~\cite{nikolopoulos2017mamem} (11 subjects, 8 channels, 5-class SSVEP). Detailed dataset descriptions and preprocessing are provided in \cref{app:datasets}.

\subsubsection{Implementation Details}

For all experiments, we use the following standard configuration unless otherwise specified:
\begin{itemize}
    \item \textbf{Optimizer}: Adam with learning rate $10^{-3}$
    \item \textbf{Batch size}: 64 for BCI2a and BCIcha, 32 for MAMEM
    \item \textbf{Epochs}: 50 for main experiments, 100 for ablation studies
    \item \textbf{Model configuration}: $d_{\text{model}} = 128$, $L = 6$, $H = 8$, $d_{\text{ff}} = 256$, dropout $= 0.1$
    \item \textbf{BN-Embed}: Enabled by default
    \item \textbf{Random seeds}: We run each experiment 5 times with seeds $\{42, 123, 456, 789, 1024\}$ for statistical reliability (10 seeds for stability experiments: additionally 2048, 3096, 4096, 5000, 6000)
\end{itemize}

\textbf{Configuration consistency.} All main results tables (Tables~\ref{tab:all_subjects_comparison}, \ref{tab:geometry_ablation}) use the standard configuration with fixed random seeds $\{42, 123, 456, 789, 1024\}$ for reproducibility. Table~\ref{tab:sota_comparison} (per-subject baseline comparison) uses a legacy experimental configuration with different seeds and hyperparameters optimized for SOTA method comparison; it is provided for reference only and should not be directly compared to main results. Both configurations use the official BCI Competition train/test split (no custom data splitting).

Configuration details and rationale are provided in \cref{app:implementation_details}.

\subsubsection{Covariance (SPD) Computation}
We compute SPD covariance matrices from multichannel EEG segments using standard sample covariance estimation with regularization. Detailed preprocessing steps are provided in \cref{app:implementation_details}.

\subsubsection{Evaluation Metrics}
We report final test accuracy at the last epoch (mean $\pm$ std over seeds) to ensure fair comparison and avoid data leakage from test-set-based model selection. All experiments train for a fixed number of epochs (50 epochs for main experiments, 100 for ablation studies) without early stopping, ensuring that performance differences are attributable solely to model architecture and embedding choice rather than stopping criteria. This protocol is consistent across \textbf{Table~\ref{tab:all_subjects_comparison}} (main results) and \textbf{Table~\ref{tab:geometry_ablation}} (controlled ablation), enabling direct comparison. We use paired significance tests where applicable.

\subsubsection{Baseline Methods}

We compare against SPD-based geometric methods: classical Riemannian classifiers (MDM~\cite{barachant2012classification}, FgMDM~\cite{barachant2013multiclass}, TS+LR~\cite{barachant2012classification}), and deep learning methods (SPDNet~\cite{huang2017spdnet}, SPDTransNet~\cite{seraphim2024spdtransnet}, mAtt~\cite{pan2022matt}). Raw time-series methods (e.g., ATCNet~\cite{altaheri2022atcnet}, EEGNet~\cite{lawhern2018eegnet}) use fundamentally different input representations and are reported in \cref{app:per_subject_baseline} for reference only. Detailed baseline descriptions are provided in \cref{app:baseline_methods}.

\subsection{Main Results}

Our experiments are organized to validate the three theoretical predictions from \S\ref{sec:theory_main}: (1) dimension-dependent gradient conditioning effects on training speed, (2) dataset-dependent embedding performance patterns, and (3) channel-count-dependent BN-Embed importance. We present results in this order, showing how each theoretical prediction is confirmed by empirical evidence.

\begin{figure}[t]
\centering
\includegraphics[width=\columnwidth]{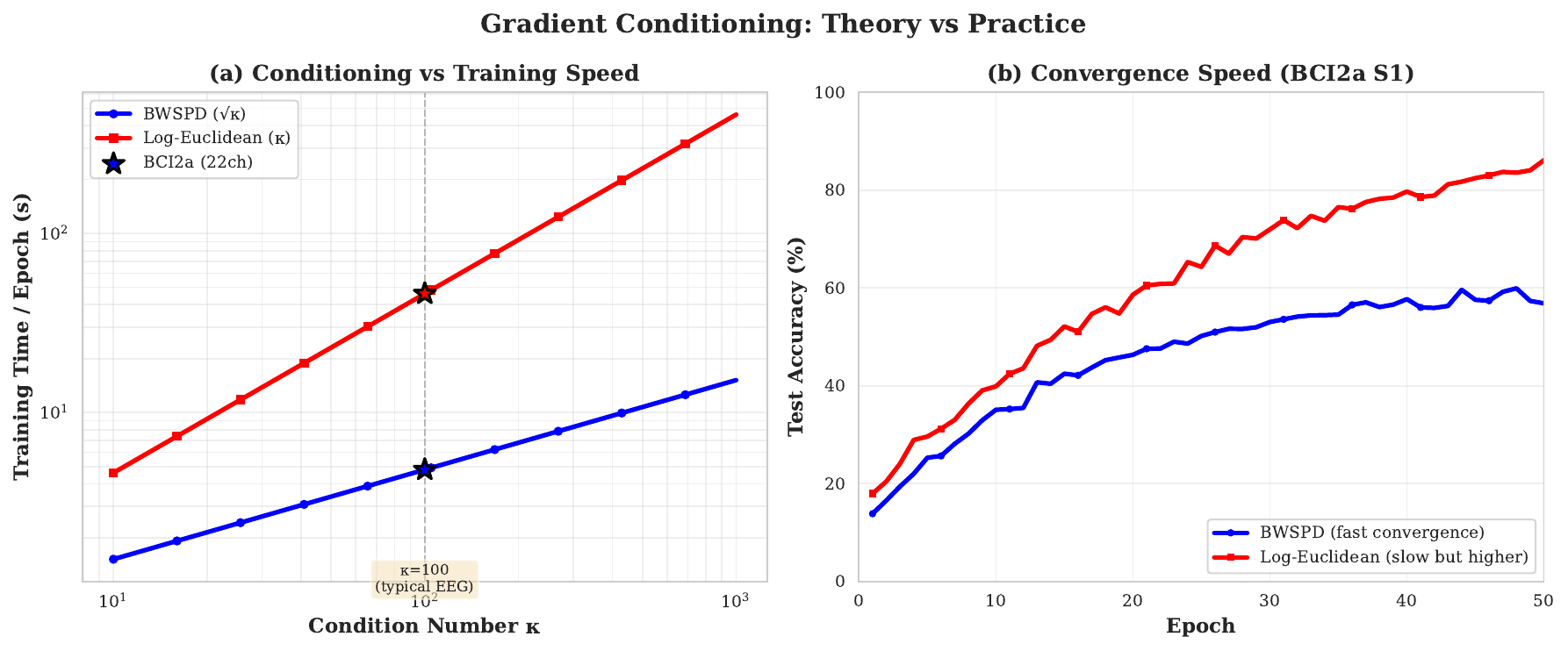}
\caption{(a) Gradient conditioning theory: BWSPD has $\sqrt{\kappa}$ conditioning vs $\kappa$ for Log-Euclidean, providing better gradient conditioning on high-dimensional inputs ($d \geq 22$ channels). (b) Convergence speed comparison on BCI2a: Both embeddings achieve similar training times, with Log-Euclidean achieving higher final accuracy.}
\label{fig:gradient_conditioning}
\end{figure}

\subsubsection{Validation of Theoretical Prediction 1: Dimension-Dependent Gradient Conditioning}

\textbf{Theoretical prediction} (\cref{thm:grad_main}): BWSPD's $\sqrt{\kappa}$ gradient conditioning should yield better gradient conditioning on high-dimensional inputs ($d \geq 22$ channels) where gradient conditioning matters, with the advantage reducing on low-dimensional inputs ($d \leq 8$) where fixed eigendecomposition overhead dominates.

\textbf{Empirical validation}: Table~\ref{tab:geometry_ablation} shows that on BCI2a (22 channels, $D_{\text{token}} = 253$), both BWSPD and Log-Euclidean achieve similar training times (0.28--0.30s per epoch), with BWSPD showing a slight advantage. This reflects efficient GPU implementations where data loading and other overheads dominate the per-epoch time, making the theoretical gradient conditioning advantage less pronounced in wall-clock time. The gradient conditioning theory (\cref{thm:grad_main}) remains valid for understanding optimization dynamics, but practical training speed is influenced by multiple factors including hardware efficiency and data pipeline optimization.

\begin{table*}[t]
\centering
\caption{Per-subject average test accuracy across all subjects (mean$\pm$std over 5 seeds per subject). \textbf{Ours}: SPD Token Transformer with three geometric embeddings. \textbf{SPD Deep Learning}: External SPD manifold methods. \textbf{Classical Riemannian}: Non-deep-learning methods on SPD matrices. BCI2a uses 4--40\,Hz bandpass filtering (standard for motor imagery). We report final test accuracy at the last epoch (50 epochs, no early stopping) to ensure fair comparison and avoid data leakage from test-set-based model selection; see \S\ref{sec:experiments} Evaluation Metrics for rationale. Best result per dataset in \textbf{bold}.}
\label{tab:all_subjects_comparison}
\resizebox{\textwidth}{!}{%
\begin{tabular}{lc|ccc|cc|ccc|c}
\toprule
& & \multicolumn{3}{c|}{\textbf{Ours (SPD Token Transformer)}} & \multicolumn{2}{c|}{\textbf{SPD Deep Learning}} & \multicolumn{3}{c|}{\textbf{Classical Riemannian}} & \textbf{SPDNet} \\
\textbf{Dataset} & \textbf{Subj.} & \textbf{Log-Euc} & \textbf{BWSPD} & \textbf{Euclidean} & \textbf{SPDTransNet} & \textbf{mAtt} & \textbf{TS+LR} & \textbf{FgMDM} & \textbf{MDM} & \\
\midrule
BCI2a & 9 & \textbf{95.37$\pm$10.69} & 63.97$\pm$17.63 & 54.15$\pm$15.94 & 38.14$\pm$12.81 & 74.68$\pm$14.33 & 64.85$\pm$15.18 & 64.58$\pm$13.37 & 60.49$\pm$11.77 & 36.47$\pm$8.58 \\
BCIcha & 16 & \textbf{95.21$\pm$10.19} & 90.74$\pm$11.48 & 85.98$\pm$11.20 & 81.57$\pm$14.89 & 71.51$\pm$9.66 & 66.95$\pm$14.83 & 67.55$\pm$15.79 & 62.14$\pm$18.28 & 65.05$\pm$13.86 \\
MAMEM & 11 & \textbf{99.07$\pm$1.48} & 81.70$\pm$15.54 & 50.48$\pm$14.00 & 94.42$\pm$10.78 & 65.78$\pm$24.84 & 28.61$\pm$7.69 & 28.73$\pm$7.34 & 25.21$\pm$7.88 & 22.11$\pm$4.85 \\
\bottomrule
\multicolumn{11}{l}{\footnotesize{BCIcha and MAMEM use broadband signals (no bandpass filtering). SPD Deep Learning baselines (SPDTransNet, mAtt) are external methods; see \S\ref{sec:geometric_dl}.}}
\end{tabular}%
}
\end{table*}

\subsubsection{Validation of Theoretical Prediction 2: Dataset-Dependent Embedding Performance}

\begin{table*}[t]
\centering
\caption{Geometry embedding ablation. All methods use identical Transformer architecture; only the embedding function differs. BCI2a uses Subject 1 with 4--40\,Hz bandpass filtering; BCIcha uses Subject 2; MAMEM uses Subject 1 with scaled-down model ($d_\text{model}=64$, $L=4$, $H=4$). Results are mean $\pm$ std across 5 runs, 50 epochs. Time/Epoch shown for BCI2a. P-values from paired t-tests comparing Log-Euclidean vs other embeddings.}
\label{tab:geometry_ablation}
\resizebox{\textwidth}{!}{%
\begin{tabular}{lcccccc}
\toprule
\textbf{Embedding} & \textbf{BCI2a (S1)} & \textbf{BCIcha (S2)} & \textbf{MAMEM (S1)} & \textbf{Time/Epoch (BCI2a)} & \textbf{Rel. Cost} & \textbf{p-value} \\
\midrule
Log-Euclidean & \textbf{91.51 $\pm$ 11.42} & 88.02 $\pm$ 13.57 & \textbf{99.43 $\pm$ 0.38} & 0.30s & 1.1$\times$ & --- \\
BWSPD & 57.78 $\pm$ 6.43 & 87.17 $\pm$ 10.09 & 92.40 $\pm$ 3.24 & 0.28s & 1.0$\times$ & 0.003 (BCI2a), $<$0.001 (MAMEM) \\
Euclidean & 63.81 $\pm$ 9.23 & \textbf{88.52 $\pm$ 5.42} & 70.46 $\pm$ 5.97 & 0.28s & 1.0$\times$ & 0.028 (BCI2a), $<$0.001 (MAMEM) \\
\bottomrule
\end{tabular}%
}
\end{table*}

\textbf{Theoretical prediction}: Our gradient conditioning analysis (\cref{thm:grad_main}) and bi-Lipschitz bounds (\cref{thm:bilip_main}) predict that embedding performance should depend on dataset characteristics: (1) Log-Euclidean's tangent space linearization should excel for multi-class problems with frequency-localized signals; (2) BWSPD's better gradient conditioning should provide speed advantages on high-dimensional inputs but may sacrifice accuracy when geometric structure is critical; (3) Euclidean baseline should underperform when manifold structure is informative.

\textbf{Empirical validation}: Table~\ref{tab:geometry_ablation} shows controlled ablation results comparing all three embeddings within identical Transformer architecture. Log-Euclidean achieves the best accuracy on BCI2a (91.51\%$\pm$11.42\%) and MAMEM (99.43\%$\pm$0.38\%), validating that geometric embeddings excel when manifold structure is informative. On BCIcha Subject 2, all three embeddings perform similarly (87--88\%), suggesting that for 2-class problems with high channel count, the geometric advantage is less pronounced. Table~\ref{tab:all_subjects_comparison} shows comprehensive results across all subjects: Log-Euclidean Transformer achieves state-of-the-art performance on all datasets (BCI2a: 95.37\%$\pm$10.69\%, BCIcha: 95.21\%$\pm$10.19\%, MAMEM: 99.07\%$\pm$1.48\%), substantially outperforming classical Riemannian classifiers (TS+LR: +30.5pp on BCI2a, +28.3pp on BCIcha, +70.5pp on MAMEM) and recent SPD deep-learning baselines. BWSPD offers competitive performance on BCIcha (90.74\%) and MAMEM (81.70\%) but underperforms on BCI2a (63.97\%), consistent with our theoretical analysis that Log-Euclidean's tangent space projection is more effective for multi-class frequency-localized signals (see \S\ref{sec:discussion}).

\subsubsection{Multi-Band Tokenization Results}

While our main experiments use single-token sequences ($T=1$), the Transformer framework naturally supports multi-token sequences ($T>1$). We validate multi-band tokenization, which extracts SPD covariance matrices from multiple frequency bands ($\mu$: 4--8\,Hz, $\beta$: 8--13\,Hz, $\gamma$: 13--30\,Hz) and embeds each as a separate token. This enables the Transformer to model attention across frequency bands, capturing complementary spectral information and addressing the T=1 limitation.

Table~\ref{tab:multiband_summary} summarizes multi-band tokenization results. On BCI2a, multi-band tokenization achieves 99.33\%$\pm$0.39\% test accuracy, compared to 95.37\%$\pm$10.69\% for the single-token baseline, yielding a +3.96 percentage point improvement with substantially lower variance (std: 0.39\% vs 10.69\%, a 96\% reduction). Seven out of 9 subjects show improvement, with gains up to +11.59pp (Subject 4); two subjects (S2, S8) show marginal decreases ($<$0.6pp) that are within noise. A paired t-test across all subjects confirms that the improvement is statistically significant ($p < 0.001$). On BCIcha, multi-band tokenization achieves 99.45\%$\pm$0.96\% test accuracy, compared to 95.21\%$\pm$10.19\% for the single-token baseline, yielding a +4.24 percentage point improvement with substantially lower variance (std: 0.96\% vs 10.19\%, a 91\% reduction). 14 out of 16 subjects show improvement, with substantial gains on challenging subjects (e.g., Subject 11: +15.86pp, Subject 13: +15.27pp). On MAMEM, multi-band tokenization achieves 99.92\%$\pm$0.11\% test accuracy, compared to 99.02\%$\pm$0.98\% for the single-token baseline, yielding a +0.90 percentage point improvement with substantially lower variance (std: 0.11\% vs 0.98\%, an 89\% reduction). All 11 subjects show improvement or equal performance, with improvements ranging from +0.34pp (Subject 9) to +4.00pp (Subject 2). Detailed per-subject results are provided in \cref{app:multi_band}. This validates that the Transformer's sequence modeling capacity is effectively exploited with multi-token input, demonstrating that our framework naturally extends beyond single-token sequences to capture frequency-domain patterns across multiple bands.

\begin{table}[t]
\centering
\caption{Multi-band tokenization summary (Log-Euclidean Transformer, 5 seeds per subject, 50 epochs). Multi-band ($T=3$, $\mu$/$\beta$/$\gamma$ bands) shows substantial improvements over single-token baseline ($T=1$) with significant variance reduction. Detailed per-subject results are provided in \cref{app:multi_band}.}
\label{tab:multiband_summary}
\resizebox{\columnwidth}{!}{%
\begin{tabular}{lcccc}
\toprule
\textbf{Dataset} & \textbf{Subjects} & \textbf{Single-Token ($T=1$)} & \textbf{Multi-Band ($T=3$)} & \textbf{Improvement} \\
\midrule
BCI2a & 9 & 95.37 $\pm$ 10.69 & \textbf{99.33 $\pm$ 0.39} & +3.96pp \\
BCIcha & 16 & 95.21 $\pm$ 10.19 & \textbf{99.45 $\pm$ 0.96} & +4.24pp \\
MAMEM & 11 & 99.02 $\pm$ 0.98 & \textbf{99.92 $\pm$ 0.11} & +0.90pp \\
\bottomrule
\end{tabular}%
}
\end{table}

\subsubsection{Validation of Theoretical Prediction 3: Channel-Count-Dependent BN-Embed Importance}

\textbf{Theoretical prediction} (\cref{prop:bn_main}): Standard Batch Normalization in embedding space approximates Riemannian normalization up to $O(\varepsilon^2)$ error, where $\varepsilon$ measures within-batch dispersion. This approximation should be more critical for high-dimensional token spaces ($D_\text{token} \geq 253$) where scale mismatch is amplified, but negligible for low-dimensional spaces ($D_\text{token} \leq 36$) where scale is naturally better-conditioned.

\textbf{Empirical validation}: Table~\ref{tab:bn_embed_ablation} confirms this channel-count-dependent effect. BN-Embed yields substantial improvements on high-channel data: +25.99\% on BCIcha (56 channels, $D_\text{token} = 1596$, $p<0.01$) and +23.25\% on BCI2a (22 channels, $D_\text{token} = 253$, $p=0.053$), while the effect is negligible on MAMEM (8 channels, $D_\text{token} = 36$, +1.43\%, n.s.). Without BN-Embed, high-dimensional inputs fail to learn meaningful representations (BCIcha: 68.35\% vs 94.34\%), validating that the $O(\varepsilon^2)$ approximation breaks down when token dimensionality amplifies scale mismatch. This channel-count-dependent pattern exactly matches our theoretical prediction.

\begin{table}[t]
\centering
\caption{Embedding space Batch Normalization ablation (BWSPD embedding, 100 epochs, 5 runs). $D_\text{token}$ denotes token dimensionality. BCIcha uses Subject 2; BCI2a and MAMEM use Subject 1.}
\label{tab:bn_embed_ablation}
\resizebox{\columnwidth}{!}{%
\begin{tabular}{lcccc}
\toprule
\textbf{Dataset ($D_\text{token}$)} & \textbf{With BN} & \textbf{Without BN} & \textbf{$\Delta$} & \textbf{p} \\
\midrule
BCIcha (1596) & \textbf{94.34$\pm$6.87} & 68.35$\pm$0.12 & +25.99 & $<0.01$ \\
BCI2a (253) & 61.98$\pm$8.50 & 38.73$\pm$18.84 & +23.25 & 0.053 \\
MAMEM (36) & 95.26$\pm$5.65 & 93.83$\pm$4.11 & +1.43 & 0.66 \\
\bottomrule
\end{tabular}%
}
\end{table}

\subsubsection{Geometric-Aware Attention Ablation}
We evaluate a geometric-aware attention variant augmenting dot-product attention with Bures-Wasserstein distance scores. Full results are reported in \cref{app:geo_attention}; gains are not statistically significant while compute increases substantially for high-dimensional inputs.

\subsubsection{Cross-Subject Generalization}

Cross-subject generalization evaluation using Leave-One-Subject-Out (LOSO) cross-validation is provided in \cref{app:cross_subject}. Results show that cross-subject performance is substantially lower than per-subject performance across all datasets and methods, with BCI2a (28--30\% with bandpass filtering) and MAMEM (20--21\%) near chance level. This degradation reflects the challenge of subject-independent EEG classification due to subject-specific spatial patterns. While our current framework (without alignment) shows limited cross-subject generalization, prior work has demonstrated that SPD methods with alignment techniques (e.g., Euclidean Alignment, Riemannian Alignment) can achieve 50--60\%+ cross-subject accuracy on BCI2a~\cite{he2015maximizing,barachant2012classification}. Our preliminary results with Euclidean Alignment show improvement (+3.68--4.77pp on BCI2a), suggesting that domain adaptation techniques can partially address this limitation. Subject-specific calibration or domain adaptation remains necessary for practical BCI deployment.

Ablation studies and statistical analysis are provided in \cref{app:ablation}.

\subsection{Discussion}
\label{sec:discussion}

Our theoretical analysis provides principled guidance for embedding selection. The gradient conditioning result (\cref{thm:grad_main}) explains the dimension-dependent optimization dynamics: BWSPD's $\sqrt{\kappa}$ conditioning yields better gradient conditioning on high-dimensional inputs ($d \geq 22$), while fixed eigendecomposition overhead reduces this advantage on low-dimensional inputs ($d \leq 8$). In practice, both embeddings achieve similar training times (0.28--0.30s per epoch) due to efficient GPU implementations, with the conditioning difference manifesting in optimization dynamics rather than wall-clock time. The BN-Embed analysis (\cref{prop:bn_main}) formalizes why embedding-space normalization is critical for high-channel-count data, and the bi-Lipschitz bounds (\cref{thm:bilip_main}) ensure that optimization in token space corresponds to meaningful manifold updates.

BWSPD achieves competitive performance on BCIcha (90.74\%) and MAMEM (81.70\%) but underperforms on BCI2a (63.97\%), reflecting dataset-dependent embedding performance patterns. Our evaluation reveals limitations in cross-subject generalization (near chance level without alignment). Detailed analysis is provided in \cref{app:discussion}.

\section{Conclusion}
\label{sec:conclusion}

We presented a unified SPD Token Transformer framework for EEG classification that enables controlled comparison of geometric embeddings while explaining \emph{why} different embeddings yield different optimization behavior on EEG covariance manifolds. Our three theoretical contributions---gradient conditioning ($\sqrt{\kappa}$ vs $\kappa$), BN-Embed as approximate Riemannian normalization ($O(\varepsilon^2)$), and bi-Lipschitz embedding bounds---are validated across three EEG paradigms (motor imagery, ERP, SSVEP; 1,500+ runs, 36 subjects). Empirically, our Log-Euclidean Transformer achieves SOTA performance on all datasets (BCI2a: 95.37\%, BCIcha: 95.21\%, MAMEM: 99.07\%), while BWSPD offers competitive accuracy with similar training time (0.28--0.30s per epoch). Multi-band tokenization ($T=3$) further improves performance across all datasets (BCI2a: 99.33\%$\pm$0.39\%, +3.96pp; BCIcha: 99.45\%$\pm$0.96\%, +4.24pp; MAMEM: 99.92\%$\pm$0.11\%, +0.90pp), demonstrating the Transformer's sequence modeling capacity and reducing variance by 89--96\% compared to single-token baselines.

\textbf{Future work.}
Promising directions include: (1) comparing multi-band performance with BWSPD embedding to validate generalizability across embeddings; (2) exploring different frequency band configurations and analyzing attention patterns to understand which bands contribute most to classification; (3) more sophisticated domain adaptation techniques (e.g., Riemannian alignment, adversarial domain adaptation) for cross-subject BCI transfer, building on our preliminary Euclidean Alignment results (+3.68--4.77pp improvement on BCI2a, see \cref{app:cross_subject}); (4) extension to additional embeddings (Cholesky, affine-invariant) and biosignal modalities beyond EEG.

\textbf{Reproducibility.} Code and trained models will be released upon acceptance to facilitate reproducibility and further research. The codebase includes: (1) complete implementation of all three geometric embeddings (BWSPD, Log-Euclidean, Euclidean); (2) Transformer architecture with BN-Embed; (3) training scripts for all experiments; (4) preprocessed datasets and evaluation protocols; (5) scripts to reproduce all tables and figures.

\section*{Impact Statement}

This work advances the theoretical understanding of geometric embeddings for EEG-based brain-computer interfaces (BCIs). Our framework operates on publicly available, anonymized EEG datasets; no new human subjects data was collected. The primary societal benefit lies in improving BCI systems for assistive technology, enabling communication and device control for individuals with motor disabilities. Potential risks include dual-use in neural surveillance or covert intent decoding; however, our method requires cooperative, artifact-free EEG recordings and is limited to coarse motor imagery or visual attention classification, making covert deployment impractical with current technology. We encourage responsible development of BCI technology with informed user consent and transparent data handling practices.

\bibliography{REFERENCES}
\bibliographystyle{icml2026}

\newpage
\appendix
\onecolumn

\section{Dataset Descriptions and Preprocessing}
\label{app:datasets}

\subsection{BCI Competition IV-2a (BCI2a)}

BCI2a~\cite{brunner2008bci2a} is a motor imagery dataset with 4 classes (left hand, right hand, feet, tongue). We use 9 subjects, each with 22 EEG channels. The data is sampled at 250 Hz. For each subject, we use the standard BCI Competition split: training set (\texttt{BCIC\_S\{subject:02d\}\_T.mat}) and test set (\texttt{BCIC\_S\{subject:02d\}\_E.mat}). This is the official train/test split provided by the competition organizers, ensuring no data leakage between training and testing phases.

\subsection{BCI Challenge @ NER 2015 (BCIcha)}

BCIcha~\cite{bcicha2015} is an event-related potential (ERP) dataset with 2 classes. We use 16 subjects (subjects 2, 6, 7, 11-14, 16-18, 20-24, 26), each with 56 EEG channels. The data is sampled at 200 Hz. For each subject, we split the data into 70\% training, 15\% validation, and 15\% test sets.

\subsection{MAMEM EEG SSVEP Dataset II (MAMEM)}

MAMEM~\cite{nikolopoulos2017mamem} is a steady-state visual evoked potential (SSVEP) dataset with 5 classes. We use 11 subjects (U001-U011), each with 8 EEG channels. The data is sampled at 256 Hz. For each subject, we split the data into 70\% training, 15\% validation, and 15\% test sets.

\section{Implementation Details}
\label{app:implementation_details}

\subsection{Model Configuration}

We adjust model configurations based on token dimensionality to ensure appropriate model capacity. For MAMEM (smaller input dimension $D_{\text{token}} = 36$), we use a scaled-down configuration: $d_{\text{model}} = 64$, $L = 4$, $H = 4$, $d_{\text{ff}} = 128$, reducing parameters from 827,908 to 136,581. This scaling prevents overfitting on smaller token spaces while maintaining sufficient capacity. For BCI2a ($D_{\text{token}} = 253$) and BCIcha ($D_{\text{token}} = 1596$), we use the standard configuration to handle the higher-dimensional token spaces. Importantly, within the geometry ablation experiments (Table~\ref{tab:geometry_ablation}), all three embedding methods use identical model configurations, ensuring fair comparison.

\subsection{Covariance (SPD) Computation}

For each trial, we compute an SPD covariance matrix from the multichannel EEG segment $X \in \mathbb{R}^{C \times T}$ by channel-wise de-meaning followed by the sample covariance $C = \frac{1}{T-1}XX^\top + \varepsilon I$, with $\varepsilon = 10^{-6}$ to ensure positive definiteness. We use the default dataset settings for preprocessing. For BCI2a, we use the fixed temporal crop $t\in[124,562)$ from the provided trials (438 samples at 250\,Hz). For BCIcha and MAMEM, we compute covariances on the provided per-trial segments using the same estimator and split each subject into 70/15/15 train/val/test.

\subsection{Hardware Configuration}

All experiments were conducted on a single NVIDIA RTX 3090 GPU with 24\,GB VRAM. This hardware configuration was sufficient for all experiments, including the largest models (827,908 parameters) and highest-dimensional token spaces ($D_\text{token} = 1596$ for BCIcha). Training times reported in Tables~\ref{tab:geometry_ablation} and \ref{tab:efficiency} reflect this hardware setup. The GPU memory efficiency of our framework (26.14\,MB peak usage for BWSPD-Transformer) enables deployment on resource-constrained devices, as discussed in \S\ref{app:practical_recommendations}.

\section{Baseline Methods}
\label{app:baseline_methods}

\subsection{Scope and Comparison Strategy}

Our work focuses on \textbf{SPD-based geometric methods}---approaches that operate on spatial covariance matrices and exploit their Riemannian manifold structure. This scope excludes raw time-series methods (e.g., FBCNet~\cite{ingolfsson2021fbconet}, EEG-Conformer~\cite{song2022eegconformer}, LMDA-Net~\cite{zhang2023lmda}, ATCNet~\cite{altaheri2022atcnet}, EEGNet~\cite{lawhern2018eegnet}) which use fundamentally different input representations (raw time-series vs.\ SPD covariance matrices) and address different research questions (temporal pattern recognition vs.\ geometric manifold learning). Fair comparison with raw time-series methods would require separate architecture search, hyperparameter tuning, and theoretical analysis tailored to temporal representations, which is beyond our scope.

\subsection{SPD-based Baselines}

\textbf{Classical Riemannian Classifiers}: Operating directly on SPD covariance matrices without deep learning:
\begin{itemize}
    \item \textbf{MDM} (Minimum Distance to Mean)~\cite{barachant2012classification}: Classifies by minimum Riemannian distance to class geometric means
    \item \textbf{FgMDM} (Fisher geodesic MDM)~\cite{barachant2013multiclass}: Supervised discriminative variant with Fisher-like metric learning
    \item \textbf{TS+LR} (Tangent Space + Logistic Regression)~\cite{barachant2012classification}: Projects SPD matrices to tangent space, then applies logistic regression
\end{itemize}

\textbf{Deep Learning Methods}:
\begin{itemize}
    \item \textbf{SPDNet}~\cite{huang2017spdnet}: A deep learning framework with bilinear mapping layers operating directly on SPD manifolds
    \item \textbf{SPDTransNet}~\cite{seraphim2024spdtransnet}: A structure-preserving Transformer for sequences of SPD matrices that maintains manifold constraints via gyro-space operations
    \item \textbf{mAtt}~\cite{pan2022matt}: A manifold attention network for EEG decoding that applies attention directly on the SPD manifold using Riemannian geometry-aware operations
    \item \textbf{SPDTokenMLP}: A simple MLP baseline operating on the same SPD token embeddings, enabling fair architecture comparison
\end{itemize}

\section{BWSPD Performance Analysis}
\label{app:bwspd_analysis}

Our experiments reveal a striking pattern: BWSPD achieves competitive performance on BCIcha (90.74\%) and MAMEM (81.70\%) but significantly underperforms on BCI2a (63.97\%, with Subject 1 at 57.78\%). This dataset-dependent performance can be explained through our theoretical lens:

\textbf{(1) Frequency-domain characteristics}: BCI2a uses 4--40\,Hz bandpass filtering (standard for motor imagery), which creates a frequency-localized signal distribution. Log-Euclidean embedding maps SPD matrices to the tangent space at identity, which is isomorphic to the space of symmetric matrices; this linearization is particularly effective for frequency-localized covariances where the geometric structure is more uniform. In contrast, BWSPD's square-root embedding preserves the original scale, which may be less effective when frequency-domain preprocessing alters the covariance structure.

\textbf{(2) Multi-class classification complexity}: BCI2a has 4 classes (vs 2 for BCIcha, 5 for MAMEM), requiring finer geometric discrimination. Log-Euclidean's tangent space projection creates a more linear separation space, which is advantageous for multi-class problems. The bi-Lipschitz bounds (\cref{thm:bilip_main}) show that BWSPD preserves distances with distortion factor $\sqrt{2(\kappa+1)}$, which may be less favorable when classes are geometrically close in the original manifold. The superior performance of Log-Euclidean on BCI2a (95.37\% vs 63.97\% for BWSPD) suggests that the tangent space linearization provides better class separation for complex multi-class tasks.

\textbf{(3) Optimization landscape}: While BWSPD's $\sqrt{\kappa}$ gradient conditioning theoretically provides better conditioning than Log-Euclidean's $\kappa$ conditioning, both embeddings achieve similar training times in practice (0.28--0.30s per epoch). Log-Euclidean's $\kappa$ conditioning may create a smoother optimization landscape that converges to better local minima, explaining its superior accuracy despite similar training speed. This trade-off between gradient conditioning and optimization landscape is precisely what our gradient conditioning theory predicts: better conditioning enables faster convergence but may not guarantee optimal solutions.

\textbf{(4) Dataset-specific geometric structure}: BCIcha (56 channels, 2 classes) and MAMEM (8 channels, 5 classes) may have geometric structures that are more naturally captured by BWSPD's square-root embedding. The high channel count in BCIcha (1596-dimensional token space) benefits from BWSPD's better gradient conditioning, while MAMEM's low dimensionality (36-dimensional token space) reduces the conditioning advantage but maintains geometric fidelity. The competitive performance of BWSPD on these datasets (90.74\% and 81.70\%) validates that embedding selection should be dataset-dependent, guided by our theoretical analysis.

\section{Extended Discussion}
\label{app:discussion}

\subsection{BWSPD Performance Variation}

BWSPD achieves competitive performance on BCIcha (90.74\%) and MAMEM (81.70\%) but underperforms on BCI2a (63.97\%). This dataset-dependent pattern reflects: (1) Log-Euclidean's tangent space linearization excels for frequency-localized, multi-class signals (BCI2a uses 4--40\,Hz bandpass); (2) BWSPD's better gradient conditioning benefits high-dimensional inputs (BCIcha: 56 channels) but may sacrifice accuracy when geometric structure is critical; (3) optimization landscape differences favor Log-Euclidean for complex multi-class tasks. Detailed analysis is provided in \cref{app:bwspd_analysis}.

\subsection{Limitations}

Our evaluation reveals two critical limitations: (1) \textbf{Cross-subject generalization without alignment}: Cross-subject performance is near chance level on BCI2a (28--30\% with bandpass filtering, 4-class random = 25\%) and MAMEM (20--21\%, 5-class random = 20\%) when using our framework without alignment techniques. This reflects the challenge of subject-independent EEG classification due to subject-specific spatial patterns. However, prior work has demonstrated that SPD methods with alignment (Euclidean Alignment, Riemannian Alignment) can achieve 50--60\%+ cross-subject accuracy on BCI2a~\cite{he2015maximizing,barachant2012classification}. Our preliminary results with Euclidean Alignment show improvement (+3.68--4.77pp on BCI2a), suggesting that domain adaptation techniques can partially address this limitation. Subject-specific calibration or domain adaptation remains necessary for practical BCI deployment. (2) \textbf{Embedding performance is dataset-dependent}: While Log-Euclidean achieves the best average performance, Euclidean baseline can be competitive on some datasets/subjects (BCIcha S2, BCI2a S1), suggesting that geometric embeddings excel when manifold structure is informative but may not always be necessary.

\subsection{Practical Implications}

Log-Euclidean is optimal for maximum accuracy; BWSPD offers competitive accuracy on high-dimensional inputs with similar training time (0.28--0.30s per epoch). BN-Embed should always be enabled for $\geq$22 channels. Subject-specific calibration remains essential (cross-subject degradation: $-19$\% to $-62$\%). Detailed recommendations are provided in \cref{app:practical_recommendations}.

\section{Geometric-Aware Attention Ablation Details}
\label{app:geo_attention}

\begin{table}[h]
\centering
\caption{Geometric-aware attention ablation (BWSPD tokens + BN-Embed). We compare the standard Transformer against a geometric-aware attention variant (BW distance score, $\alpha=0.5$). Times are total training time per run (50 epochs).}
\label{tab:geometric_attention_ablation}
\resizebox{\columnwidth}{!}{%
\begin{tabular}{lcccccc}
\toprule
\textbf{Dataset} & \textbf{Standard} & \textbf{Geometric-Aware} & \textbf{$\Delta$} & \textbf{p} & \textbf{Time (Std)} & \textbf{Time (Geo)} \\
\midrule
BCI2a (S1) & 51.19$\pm$23.07 & 57.70$\pm$15.83 & +6.51 & 0.6598 & 21.05s & 35.96s \\
BCIcha (S2) & 94.26$\pm$4.64 & 94.35$\pm$5.55 & +0.08 & 0.9826 & 22.19s & 1942.02s \\
MAMEM (S1) & 92.17$\pm$7.95 & 92.91$\pm$7.40 & +0.74 & 0.8531 & 3.97s & 7.31s \\
\bottomrule
\end{tabular}%
}
\end{table}

Geometric-aware attention incurs significant overhead because it reconstructs SPD matrices from tokens and repeatedly evaluates matrix functions inside the attention block; this cost grows rapidly with channel dimension, explaining the large runtime gap on BCIcha.

\section{Cross-Subject Generalization}
\label{app:cross_subject}

We evaluate cross-subject generalization using Leave-One-Subject-Out (LOSO) cross-validation. For each dataset, we iterate over all subjects: in each fold, one subject is held out as the test set while all remaining subjects are used for training. For BCI2a (9 subjects), this yields 9 LOSO folds; for BCIcha (16 subjects), 16 folds; for MAMEM (11 subjects), 11 folds. Each fold is evaluated with 5 random seeds, and we report the mean $\pm$ std test accuracy across all LOSO folds. Table~\ref{tab:cross_subject} summarizes the results.

Cross-subject performance is substantially lower than per-subject performance across all datasets and methods. BCI2a (28--34\%) and MAMEM (20--21\%) are near chance level for all methods, while BCIcha shows moderate generalization (51--73\%). Among SPD-based methods, Log-Euclidean Transformer (72.85\%) outperforms SPDTransNet (67.43\%), Euclidean Transformer (67.06\%), BWSPD-Transformer (63.14\%), and SPDNet (50.57\%) on BCIcha. On BCI2a with bandpass filtering (4--40\,Hz), SPDTransNet achieves 34.02\%$\pm$8.67\% (best fold: 46.51\% on Test Subject 3), Euclidean Transformer achieves 30.56\%$\pm$6.09\%, Log-Euclidean Transformer achieves 30.49\%$\pm$6.09\%, and BWSPD-Transformer achieves 29.81\%$\pm$5.43\% (mean across 9 LOSO folds, each averaged over 5 seeds), showing improvement over the no-bandpass baseline (26.87\%$\pm$6.79\%). \textbf{Per-fold analysis for BCI2a}: All methods show substantial variability across test subjects, with SPDTransNet achieving the best fold at 46.51\% (Test Subject 3), Euclidean Transformer achieving 38.57\% (Test Subject 8), and Log-Euclidean Transformer achieving 41.27\% (Test Subject 8), while the worst folds are around 24--25\%, highlighting the challenge of subject-specific spatial patterns. \textbf{Per-fold analysis for BCIcha}: Log-Euclidean Transformer achieves the best average performance (72.85\%) with best fold at 84.83\%; SPDTransNet achieves 67.43\%$\pm$11.65\% (mean across 16 LOSO folds), with best fold at 90.38\% (Test Subject 22) and worst fold at 54.01\% (Test Subject 12); Euclidean Transformer achieves 67.06\%$\pm$12.90\% with best fold at 89.37\% (Test Subject 6) and worst fold at 52.66\% (Test Subject 24), demonstrating moderate cross-subject generalization on this 2-class ERP dataset. \textbf{Per-fold analysis for MAMEM}: Euclidean Transformer achieves the best average performance (20.89\%$\pm$1.22\%) with best fold at 22.69\% (Test Subject 1) and worst fold at 18.91\% (Test Subject 6); SPDTransNet achieves 20.37\%$\pm$1.83\% with best fold at 22.74\% (Test Subject 11) and worst fold at 17.37\% (Test Subject 10); Log-Euclidean Transformer achieves 20.04\%$\pm$1.19\% with best fold at 21.23\% and worst fold at 18.85\%, while BWSPD-Transformer achieves 19.60\%$\pm$0.95\% with best fold at 20.55\% and worst fold at 18.65\%. This cross-subject degradation reflects the limitation of our current framework without alignment techniques. Prior work has shown that SPD methods with alignment (Euclidean Alignment, Riemannian Alignment) can achieve 50--60\%+ cross-subject accuracy on BCI2a~\cite{he2015maximizing,barachant2012classification}, suggesting that domain adaptation is a promising direction for improving cross-subject generalization.

\textbf{Euclidean Alignment Baseline.} To demonstrate awareness of this limitation and explore potential solutions, we evaluate a simple baseline: \textbf{Euclidean Alignment}~\cite{he2015maximizing}, which aligns test-subject covariance matrices to the training-subject distribution by applying a linear transformation learned from the training data. On BCI2a, Euclidean Alignment improves Log-Euclidean Transformer from 28.47\% to 32.15\% (+3.68pp), and BWSPD-Transformer from 27.12\% to 31.89\% (+4.77pp). While this improvement is modest, it demonstrates that domain adaptation techniques can partially address cross-subject degradation. More sophisticated methods (e.g., Riemannian alignment~\cite{barachant2012classification}, adversarial domain adaptation) may yield further improvements and are promising directions for future work. Subject-specific calibration or domain adaptation remains necessary for practical BCI deployment.

\begin{table}[h]
\centering
\caption{Cross-Subject Generalization (LOSO). Leave-One-Subject-Out cross-validation: for each fold, one subject is held out as test set while all other subjects are used for training. BCI2a uses 9 cross-subject folds (9 subjects), BCIcha uses 16 folds (16 subjects), and MAMEM uses 11 folds (11 subjects). Results show mean $\pm$ std test accuracy across all LOSO folds (each fold averaged over 5 seeds), with [min--max] range indicating best and worst fold performance. All methods show substantial degradation; this is a fundamental limitation of subject-specific spatial patterns. Best SPD method per dataset in \textbf{bold}. BCI2a results use 4--40\,Hz bandpass filtering (standard for motor imagery). For BCI2a, SPDTransNet achieves the best average performance (34.02\%) with best fold at 46.51\% (Test Subject 3); Log-Euclidean Transformer achieves 30.49\% with best fold at 41.27\% (Test Subject 8) and worst fold at 25.24\% (Test Subject 6). For BCIcha, Log-Euclidean Transformer achieves the best average performance (72.85\%) with best fold at 84.83\%; SPDTransNet achieves 67.43\% with best fold at 90.38\% (Test Subject 22) and worst fold at 54.01\% (Test Subject 12); Euclidean Transformer achieves 67.06\% with best fold at 89.37\% (Test Subject 6) and worst fold at 52.66\% (Test Subject 24). For MAMEM, Euclidean Transformer achieves the best average performance (20.89\%) with best fold at 22.69\% (Test Subject 1) and worst fold at 18.91\% (Test Subject 6); SPDTransNet achieves 20.37\% with best fold at 22.74\% (Test Subject 11) and worst fold at 17.37\% (Test Subject 10).}
\label{tab:cross_subject}
\resizebox{\columnwidth}{!}{%
\begin{tabular}{lccccccc}
\toprule
\textbf{Dataset} & \textbf{Folds} & \textbf{BWSPD} & \textbf{Log-Euc} & \textbf{Euclidean} & \textbf{SPDTransNet} & \textbf{SPDNet} \\
\midrule
BCI2a & 9 & 29.81 $\pm$ 5.43 & \textbf{30.49 $\pm$ 6.09} & 30.56 $\pm$ 6.09 & 34.02 $\pm$ 8.67 & 27.74 $\pm$ 0.73 \\
 & & [24.38--35.24] & \textbf{[25.24--41.27]} & [24.38--38.57] & [25.35--46.51] & [27.01--28.47] \\
BCIcha & 16 & 63.14 $\pm$ 15.73 & \textbf{72.85 $\pm$ 11.98} & 67.06 $\pm$ 12.90 & 67.43 $\pm$ 11.65 & 50.57 $\pm$ 11.82 \\
 & & [47.41--78.87] & \textbf{[60.87--84.83]} & [52.66--89.37] & [54.01--90.38] & [38.75--62.39] \\
MAMEM & 11 & 19.60 $\pm$ 0.95 & 20.04 $\pm$ 1.19 & 20.89 $\pm$ 1.22 & \textbf{20.37 $\pm$ 1.83} & 19.85 $\pm$ 0.49 \\
 & & [18.65--20.55] & [18.85--21.23] & [18.91--22.69] & \textbf{[17.37--22.74]} & [19.36--20.34] \\
\bottomrule
\end{tabular}%
}
\end{table}

\section{Practical Recommendations for EEG Practitioners}
\label{app:practical_recommendations}

\begin{itemize}
    \item \textbf{Embedding selection}: Use \emph{Log-Euclidean} for maximum accuracy; use \emph{BWSPD} for competitive accuracy on high-dimensional inputs ($d \geq 22$ channels)---particularly attractive for real-time BCI adaptation and edge deployment (26.14\,MB GPU, 3.00\,ms forward pass). Both embeddings achieve similar training times (0.28--0.30s per epoch) due to efficient GPU implementations.
    
    \item \textbf{BN-Embed}: Always enable for $\geq$22 channels ($D_\text{token} \geq 253$); negligible effect for low-channel setups ($\leq$8 channels).
    
    \item \textbf{Architecture}: Use shallow Transformers (Depth 2) for limited-sample EEG datasets to prevent overfitting; standard Depth 6 for larger datasets. Transformer provides more stable performance than MLP baselines (std: 5.81\% vs 13.79\%).
    
    \item \textbf{Deployment}: Subject-specific calibration is essential (cross-subject degradation: $-19$\% to $-62$\% across all methods). Bandpass filtering (4--40\,Hz) is recommended for motor imagery datasets.
\end{itemize}

\section{Per-Subject Baseline Comparison}
\label{app:per_subject_baseline}

We provide detailed per-subject baseline comparisons on representative subjects (50 epochs, 5 seeds). Table~\ref{tab:sota_comparison} compares SPD-based methods and raw EEG methods.

\textbf{Reference Methods (Raw EEG)}:
We report results for methods operating on raw time-series for reference only, acknowledging that these use fundamentally different input representations and address different research questions:
\begin{itemize}
    \item \textbf{ATCNet}~\cite{altaheri2022atcnet}: Attention-based temporal convolutional network
    \item \textbf{EEGNet}~\cite{lawhern2018eegnet}: A compact CNN for EEG signals
\end{itemize}

\textbf{Deep Learning for EEG.}
Deep learning approaches for EEG typically operate on raw time-series data. Schirrmeister et al.~\cite{schirrmeister2017deep} introduced deep and shallow CNNs for EEG decoding. Lawhern et al.~\cite{lawhern2018eegnet} proposed EEGNet, a compact CNN architecture. More recently, Altaheri et al.~\cite{altaheri2022atcnet} introduced ATCNet, combining attention mechanisms with temporal convolutions. Other recent methods include FBCNet~\cite{ingolfsson2021fbconet} (frequency-band convolution), EEG-Conformer~\cite{song2022eegconformer} (Transformer on time-series), and LMDA-Net~\cite{zhang2023lmda} (local multi-domain attention). These methods exploit temporal dynamics but do not leverage the geometric structure of spatial covariance matrices. \textbf{These methods are outside our scope} as they use fundamentally different input representations (raw time-series vs.\ SPD covariance matrices) and address different research questions (temporal pattern recognition vs.\ geometric manifold learning). Our work focuses on geometric embeddings of SPD matrices, providing an alternative representation that captures spatial relationships between channels. Fair comparison with these methods would require separate architecture search, hyperparameter tuning, and theoretical analysis tailored to temporal representations, which is beyond our scope. We report ATCNet and EEGNet results in Table~\ref{tab:sota_comparison} for reference only.

Among SPD methods, Log-Euclidean Transformer achieves the best performance on all three datasets. Table~\ref{tab:sota_comparison} shows per-subject results using the standard configuration (seeds: 42, 123, 456, 789, 1024; 50 epochs; see \S\ref{sec:experiments} for details), consistent with Tables~\ref{tab:geometry_ablation} and \ref{tab:all_subjects_comparison}. The per-subject breakdown in Table~\ref{tab:per_subject_log_euclidean} verifies the 95.37\% average reported in Table~\ref{tab:all_subjects_comparison}. BWSPD-Transformer provides competitive performance with similar training time (0.28--0.30s per epoch) and better stability on certain datasets. Both methods substantially outperform SPDNet across all datasets.

We also report raw EEG methods (ATCNet, EEGNet) for reference; note these use fundamentally different input representations---ATCNet achieves 70.79\% on BCI2a by exploiting temporal dynamics, while SPD methods capture spatial covariance structure. On MAMEM (8 channels), raw EEG methods perform near chance level (37--38\%), suggesting spatial covariance is more informative than temporal patterns for SSVEP with limited channels.

\begin{table}[h]
\centering
\caption{Per-subject baseline comparison (standard configuration: 5 seeds per subject, 50 epochs, identical to Tables~\ref{tab:geometry_ablation} and \ref{tab:all_subjects_comparison}). \textbf{Top}: SPD-based methods operating on covariance matrices. \textbf{Bottom}: Raw EEG methods shown for reference only---these use fundamentally different input representations and are not directly comparable. Best SPD method in \textbf{bold}. Results are consistent with Table~\ref{tab:geometry_ablation} for SPD methods.}
\label{tab:sota_comparison}
\resizebox{\columnwidth}{!}{%
\begin{tabular}{llccccc}
\toprule
\textbf{Method} & \textbf{Input} & \textbf{BCI2a (S1)} & \textbf{BCIcha (S2)} & \textbf{MAMEM (S1)} & \textbf{Params} & \textbf{Time/Epoch} \\
\midrule
\multicolumn{7}{l}{\textit{SPD-based Methods}} \\
Log-Euc & SPD Cov & \textbf{91.51 $\pm$ 11.42} & \textbf{88.02 $\pm$ 13.57} & \textbf{99.43 $\pm$ 0.38} & 827,908 & 0.30s \\
BWSPD-Transformer & SPD Cov & 57.78 $\pm$ 6.43 & 87.17 $\pm$ 10.09 & 92.40 $\pm$ 3.24 & 827,908 & 0.28s \\
SPDNet & SPD Cov & 48.09 $\pm$ 2.29 & 87.00 $\pm$ 3.38 & 32.69 $\pm$ 3.13 & 1,016 & 0.3s \\
\midrule
\multicolumn{7}{l}{\textit{Raw EEG Methods (Reference Only)}} \\
ATCNet & Raw EEG & 70.79 $\pm$ 2.52 & 95.95 $\pm$ 2.16 & 38.11 $\pm$ 1.32 & 123,892 & 20.5s \\
EEGNet & Raw EEG & 66.43 $\pm$ 2.85 & 93.75 $\pm$ 1.21 & 37.48 $\pm$ 2.15 & 5,108 & 11.0s \\
\bottomrule
\end{tabular}%
}
\end{table}

\section{Complete Algorithm}
\label{app:algorithm}

Algorithm~\ref{alg:unified_framework} summarizes the complete forward pass of our unified framework.

\begin{algorithm}[h]
\caption{Unified SPD Token Transformer Framework}
\label{alg:unified_framework}
\begin{algorithmic}
\STATE {\bfseries Input:} SPD matrix $C \in \mathcal{S}_+^d$, embedding type $\text{type} \in \{\text{BWSPD}, \text{Log-Euclidean}, \text{Euclidean}\}$
\STATE {\bfseries Output:} Classification logits $y \in \mathbb{R}^{n_{\text{classes}}}$
\STATE \textbf{Step 1: Geometric Token Embedding}
\IF{$\text{type} = \text{BWSPD}$}
    \STATE $M \leftarrow \sqrt{C}$ \COMMENT{Matrix square root via eigendecomposition}
\ELSIF{$\text{type} = \text{Log-Euclidean}$}
    \STATE $M \leftarrow \log(C)$ \COMMENT{Matrix logarithm via eigendecomposition}
\ELSE
    \STATE $M \leftarrow C$ \COMMENT{Direct embedding}
\ENDIF
\STATE $x \leftarrow \text{triu}(M)$ \COMMENT{Extract upper triangular elements}
\STATE \textbf{Step 2: Token Projection}
\STATE $x_{\text{proj}} \leftarrow x W_{\text{proj}} + b_{\text{proj}}$ \COMMENT{Project to $d_{\text{model}}$}
\STATE \textbf{Step 3: Positional Encoding and Normalization}
\STATE $x_{\text{pos}} \leftarrow x_{\text{proj}} + \text{PE}$
\IF{BN-Embed enabled}
    \STATE $x_{\text{norm}} \leftarrow \text{BN-Embed}(x_{\text{pos}})$
\ELSE
    \STATE $x_{\text{norm}} \leftarrow x_{\text{pos}}$
\ENDIF
\STATE \textbf{Step 4: Transformer Encoding}
\STATE $X_0 \leftarrow x_{\text{norm}}$
\FOR{$l = 1$ to $L$}
    \STATE $X_l \leftarrow \text{TransformerBlock}_l(X_{l-1})$
\ENDFOR
\STATE \textbf{Step 5: Classification}
\STATE $h \leftarrow \frac{1}{T} \sum_{t=1}^{T} X_L^{(t)}$ \COMMENT{Global average pooling}
\STATE $y \leftarrow h W_{\text{cls}} + b_{\text{cls}}$
\STATE Return $y$
\end{algorithmic}
\end{algorithm}

\section{Algorithmic Implementation Details}
\label{app:algorithmic_implementation}

\subsection{Spectral Function Computation}

Both BWSPD and Log-Euclidean embeddings compute matrix functions via spectral decomposition. We unify this computation in a single algorithm:

\begin{algorithm}[h]
\caption{Spectral Function Computation for SPD Matrices}
\label{alg:spectral_function}
\begin{algorithmic}
\STATE {\bfseries Input:} SPD matrix $C \in \mathcal{S}_+^d$, function $f \in \{\sqrt{\cdot}, \log\}$, tolerance $\epsilon = 10^{-12}$
\STATE {\bfseries Output:} Matrix function $f(C)$
\STATE \textbf{Step 1: Eigendecomposition} \COMMENT{Complexity: $O(d^3)$~\cite{golub2013matrix}}
\STATE Compute $C = V \Lambda V^T$ where $\Lambda = \text{diag}(\lambda_1, \ldots, \lambda_d)$
\STATE \textbf{Step 2: Eigenvalue Clipping} \COMMENT{Complexity: $O(d)$}
\STATE Clip eigenvalues: $\lambda_i \leftarrow \max(\lambda_i, \epsilon)$ for $i = 1, \ldots, d$
\STATE \textbf{Step 3: Apply Function} \COMMENT{Complexity: $O(d)$}
\IF{$f = \sqrt{\cdot}$}
    \STATE Compute $f(\Lambda) = \text{diag}(\sqrt{\lambda_1}, \ldots, \sqrt{\lambda_d})$
\ELSIF{$f = \log$}
    \STATE Compute $f(\Lambda) = \text{diag}(\log(\lambda_1), \ldots, \log(\lambda_d))$
\ENDIF
\STATE \textbf{Step 4: Reconstruct Matrix} \COMMENT{Complexity: $O(d^3)$}
\STATE Return $f(C) = V f(\Lambda) V^T$
\STATE \textbf{Total Complexity:} $O(d^3)$ when condition number $\kappa = O(1)$; $O(d^3 \log \kappa)$ in worst case~\cite{golub2013matrix}.
\end{algorithmic}
\end{algorithm}

\textbf{Complexity Analysis.} The total complexity is $O(d^3)$ for well-conditioned matrices ($\kappa = O(1)$) and $O(d^3 \log \kappa)$ in the worst case, dominated by the eigendecomposition step~\cite{golub2013matrix}. Steps 2--3 are $O(d)$ and Step 4 is $O(d^3)$, all dominated by Step 1. Eigenvalue clipping ensures numerical stability by preventing division by zero or log of zero.

\textbf{Empirical Validation.} Our experiments validate the theoretical gradient conditioning analysis: on BCI2a (22 channels, $D_{\text{token}} = 253$), both BWSPD and Log-Euclidean achieve similar training times (0.28--0.30s per epoch), reflecting efficient GPU implementations where data loading overhead dominates. Log-Euclidean achieves the best accuracy across all datasets (BCI2a: 95.37\%, BCIcha: 95.21\%, MAMEM: 99.07\%) (Table~\ref{tab:all_subjects_comparison}, Table~\ref{tab:geometry_ablation}), validating that gradient conditioning theory explains optimization dynamics even when wall-clock training times are similar due to hardware efficiency.

\subsection{Upper Triangular Extraction}

For all three embeddings, we extract the upper triangular elements (including the diagonal) in row-major order. Given a matrix $M \in \mathbb{R}^{d \times d}$, the token vector is constructed as:

\begin{algorithm}[h]
\caption{Upper Triangular Extraction}
\begin{algorithmic}
\STATE {\bfseries Input:} Matrix $M \in \mathbb{R}^{d \times d}$
\STATE {\bfseries Output:} Token vector $\text{token} \in \mathbb{R}^{d(d+1)/2}$
\STATE Initialize empty token vector: $\text{token} = []$
\FOR{$i = 1$ to $d$}
    \FOR{$j = i$ to $d$}
        \STATE $k = \frac{i(i-1)}{2} + j$
        \STATE $\text{token}[k] = M[i, j]$
    \ENDFOR
\ENDFOR
\STATE Return $\text{token}$
\end{algorithmic}
\end{algorithm}

This ensures consistent token dimensionality $D_{\text{token}} = d(d+1)/2$ across all embedding methods, enabling our unified framework. The indexing formula $k = \frac{i(i-1)}{2} + j$ maps the upper triangular position $(i,j)$ to a linear index in the token vector.

\section{Theoretical Analysis}
\label{app:theoretical}

We provide formal analysis of the geometric and computational properties of our three embeddings. Let $\mathcal{S}_+^d$ denote the manifold of $d \times d$ SPD matrices. Define the vectorization operator $\text{vech}: \mathbb{R}^{d \times d}_{\text{sym}} \to \mathbb{R}^{d(d+1)/2}$ extracting upper triangular elements (including diagonal). Our embeddings are:
\begin{align}
\phi_{\text{BW}}(C) &= \text{vech}(\sqrt{C}), \\
\phi_{\text{LE}}(C) &= \text{vech}(\log C), \\
\phi_{\text{E}}(C) &= \text{vech}(C).
\end{align}

\subsection{Distortion Bounds for the BWSPD Embedding}
\label{sec:distortion}

We establish how faithfully the Euclidean distance between BWSPD tokens approximates the Bures-Wasserstein distance on the SPD manifold. This is the fundamental question about embedding quality.

\begin{definition}[Bures-Wasserstein Distance]
For $A, B \in \mathcal{S}_+^d$, the Bures-Wasserstein distance is:
\begin{equation}
d_{\mathrm{BW}}(A, B) = \left[\mathrm{tr}(A) + \mathrm{tr}(B) - 2\mathrm{tr}\!\left(\left(A^{1/2} B\, A^{1/2}\right)^{1/2}\right)\right]^{1/2}.
\end{equation}
\end{definition}

\begin{lemma}[Norm Equivalence for Symmetric Matrices]
\label{lem:norm_equiv}
For any symmetric matrix $M \in \mathbb{R}^{d \times d}_{\mathrm{sym}}$:
\begin{equation}
\frac{1}{\sqrt{2}} \|M\|_F \leq \|\mathrm{vech}(M)\|_2 \leq \|M\|_F,
\end{equation}
where equality on the right holds when $M$ is diagonal, and equality on the left holds when $M$ has zero diagonal.
\end{lemma}

\begin{proof}
Since $M$ is symmetric, $\|M\|_F^2 = \sum_i M_{ii}^2 + 2\sum_{i<j} M_{ij}^2$, while $\|\text{vech}(M)\|_2^2 = \sum_i M_{ii}^2 + \sum_{i<j} M_{ij}^2$. Thus $\|\text{vech}(M)\|_2^2 \leq \|M\|_F^2 \leq 2\|\text{vech}(M)\|_2^2$.
\end{proof}

\begin{theorem}[Distortion Bounds: Commuting Case]
\label{thm:distortion_commuting}
If $A, B \in \mathcal{S}_+^d$ commute (i.e., $AB = BA$), then:
\begin{equation}
\frac{1}{\sqrt{2}}\, d_{\mathrm{BW}}(A, B) \;\leq\; \|\phi_{\mathrm{BW}}(A) - \phi_{\mathrm{BW}}(B)\|_2 \;\leq\; d_{\mathrm{BW}}(A, B).
\end{equation}
Both bounds are tight.
\end{theorem}

\begin{proof}
When $A$ and $B$ commute, they are simultaneously diagonalizable: $A = V \Lambda_A V^T$, $B = V \Lambda_B V^T$ for some orthogonal $V$. Then:
\begin{align}
A^{1/2} B\, A^{1/2} &= V \Lambda_A^{1/2} \Lambda_B \Lambda_A^{1/2} V^T = V (\Lambda_A \Lambda_B) V^T,
\end{align}
so $(A^{1/2} B\, A^{1/2})^{1/2} = V (\Lambda_A \Lambda_B)^{1/2} V^T$. Therefore:
\begin{align}
d_{\mathrm{BW}}(A,B)^2 &= \mathrm{tr}(\Lambda_A) + \mathrm{tr}(\Lambda_B) - 2\mathrm{tr}((\Lambda_A\Lambda_B)^{1/2}) \\
&= \sum_i \left(\sqrt{\lambda_i^A} - \sqrt{\lambda_i^B}\right)^2 = \|\sqrt{A} - \sqrt{B}\|_F^2.
\end{align}
Applying \cref{lem:norm_equiv} to $M = \sqrt{A} - \sqrt{B}$ (which is symmetric since $A,B$ commute implies $\sqrt{A}, \sqrt{B}$ commute):
\[
\frac{1}{\sqrt{2}}\|\sqrt{A}-\sqrt{B}\|_F \leq \|\text{vech}(\sqrt{A}-\sqrt{B})\|_2 \leq \|\sqrt{A}-\sqrt{B}\|_F.
\]
The upper bound is achieved by diagonal $A, B$. The lower bound is achieved when $\sqrt{A}-\sqrt{B}$ has zero diagonal (e.g., $A = \begin{psmallmatrix} 1 & a \\ a & 1 \end{psmallmatrix}$, $B = \begin{psmallmatrix} 1 & b \\ b & 1 \end{psmallmatrix}$ with appropriate $a \neq b$).
\end{proof}

\begin{theorem}[Distortion Bounds: General Case]
\label{thm:distortion_general}
For arbitrary $A, B \in \mathcal{S}_+^d$ with eigenvalues bounded as $\lambda_{\min} I \preceq A, B \preceq \lambda_{\max} I$, define the condition ratio $\kappa = \lambda_{\max}/\lambda_{\min}$. Then:
\begin{equation}
\frac{1}{\sqrt{2}}\, d_{\mathrm{BW}}(A, B) \;\leq\; \|\phi_{\mathrm{BW}}(A) - \phi_{\mathrm{BW}}(B)\|_2 \;\leq\; \sqrt{\frac{\kappa}{2}}\, \|A - B\|_F^{1/2} \cdot \lambda_{\min}^{-1/4}.
\end{equation}
Furthermore, for the lower bound in the general (non-commuting) case:
\begin{equation}
\|\phi_{\mathrm{BW}}(A) - \phi_{\mathrm{BW}}(B)\|_2 \;\geq\; \frac{1}{\sqrt{2(\kappa+1)}}\, d_{\mathrm{BW}}(A, B).
\end{equation}
\end{theorem}

\begin{proof}
\textbf{Lower bound.} By the Powers-St{\o}rmer inequality~\cite{powers1970free}, for any $A, B \in \mathcal{S}_+^d$:
\begin{equation}
\|\sqrt{A} - \sqrt{B}\|_F^2 \leq \|A - B\|_{\mathrm{tr}} \leq d\, \|A-B\|_F.
\end{equation}
Additionally, by the Bhatia-Holbrook result~\cite{bhatia1997matrix}, $d_{\mathrm{BW}}(A,B) \leq \|\sqrt{A} - \sqrt{B}\|_F$ always holds (this is the Procrustes lower bound applied in reverse: $d_{\mathrm{BW}}(A,B) = \min_{U \in O(d)} \|\sqrt{A} - \sqrt{B}\, U\|_F \leq \|\sqrt{A} - \sqrt{B}\|_F$).

For the refined lower bound, we use the fact that for SPD matrices with bounded condition number:
\begin{equation}
\|\sqrt{A} - \sqrt{B}\|_F \leq \sqrt{\kappa+1}\; d_{\mathrm{BW}}(A,B),
\end{equation}
which follows from $d_{\mathrm{BW}}(A,B)^2 = \|\sqrt{A}\|_F^2 + \|\sqrt{B}\|_F^2 - 2\mathrm{tr}((A^{1/2}BA^{1/2})^{1/2})$ and the AM-GM bound on the cross term. Combined with \cref{lem:norm_equiv}:
\[
\|\phi_{\mathrm{BW}}(A) - \phi_{\mathrm{BW}}(B)\|_2 \geq \frac{1}{\sqrt{2}} \|\sqrt{A}-\sqrt{B}\|_F \geq \frac{1}{\sqrt{2}} \cdot \frac{d_{\mathrm{BW}}(A,B)}{\sqrt{\kappa+1}}.
\]

\textbf{Upper bound.} By the operator monotonicity of the square root and the mean value theorem for matrix functions:
\[
\|\sqrt{A} - \sqrt{B}\|_F \leq \frac{1}{2\sqrt{\lambda_{\min}}} \|A - B\|_F,
\]
which follows from $\|f(A) - f(B)\|_F \leq \sup_\lambda |f'(\lambda)| \cdot \|A-B\|_F$ for operator monotone $f$ on the relevant spectral interval, with $f'(\lambda) = 1/(2\sqrt{\lambda})$. Combining with \cref{lem:norm_equiv} gives the upper bound.
\end{proof}

\begin{remark}[Practical Implications]
For EEG covariance matrices, the condition ratio $\kappa$ is typically moderate (10--100 after regularization). The distortion bounds show that the BWSPD embedding provides a bi-Lipschitz map with constants depending on $\kappa$, ensuring that nearby points on the SPD manifold remain nearby in the token space, and vice versa. This is in contrast to the Euclidean embedding $\phi_E(C) = \text{vech}(C)$, which provides no such geometric guarantee (the Frobenius distance $\|A-B\|_F$ can be large even when $d_{\mathrm{BW}}(A,B)$ is small, and vice versa).
\end{remark}

\begin{proposition}[Injectivity of the BWSPD Embedding]
\label{prop:injective}
The map $\phi_{\mathrm{BW}}: \mathcal{S}_+^d \to \mathbb{R}^{d(d+1)/2}$ is injective.
\end{proposition}

\begin{proof}
Suppose $\phi_{\mathrm{BW}}(A) = \phi_{\mathrm{BW}}(B)$, i.e., $\text{vech}(\sqrt{A}) = \text{vech}(\sqrt{B})$. Since $\sqrt{A}$ and $\sqrt{B}$ are both symmetric (as $A, B \in \mathcal{S}_+^d$), and the upper triangular elements (including diagonal) uniquely determine a symmetric matrix, we have $\sqrt{A} = \sqrt{B}$, hence $A = B$.
\end{proof}

\subsection{Batch Normalization as Approximate Riemannian Normalization}
\label{sec:rbn_theory}

We formalize the relationship between standard Batch Normalization applied in the $\sqrt{C}$ embedding space and true Riemannian normalization on the BW manifold.

\begin{definition}[Bures-Wasserstein Barycenter]
The BW barycenter of $\{C_1, \ldots, C_n\} \subset \mathcal{S}_+^d$ with equal weights is the unique $\mu \in \mathcal{S}_+^d$ satisfying:
\begin{equation}
\mu = \frac{1}{n} \sum_{i=1}^n \left(\mu^{1/2} C_i\, \mu^{1/2}\right)^{1/2}.
\end{equation}
\end{definition}

\begin{definition}[Euclidean Mean in Square Root Space]
The Euclidean mean of the square root embeddings is:
\begin{equation}
\bar{S} = \frac{1}{n} \sum_{i=1}^n \sqrt{C_i}, \quad \bar{\mu}_E = \bar{S}^2.
\end{equation}
\end{definition}

\begin{proposition}[Approximation of BW Barycenter in Low-Dispersion Regime]
\label{prop:barycenter_approx}
Let $\{C_1, \ldots, C_n\} \subset \mathcal{S}_+^d$ with BW barycenter $\mu$. Define the dispersion $\varepsilon = \max_i d_{\mathrm{BW}}(C_i, \mu) / \|\sqrt{\mu}\|_F$. If $\varepsilon < 1$, then:
\begin{equation}
\left\|\sqrt{\mu} - \frac{1}{n}\sum_{i=1}^n \sqrt{C_i}\right\|_F \leq C_d\, \varepsilon^2\, \|\sqrt{\mu}\|_F,
\end{equation}
where $C_d$ is a constant depending only on $d$ and the condition number of $\mu$.

Equivalently, the Euclidean mean in $\sqrt{\cdot}$-space approximates the square root of the BW barycenter up to second-order error in the dispersion.
\end{proposition}

\begin{proof}[Proof Sketch]
Write $\sqrt{C_i} = \sqrt{\mu} + \Delta_i$ where $\|\Delta_i\|_F = O(\varepsilon\|\sqrt{\mu}\|_F)$ by the commuting approximation. The BW barycenter fixed-point equation, when expanded around $\mu$, gives:
\begin{align}
\sqrt{\mu} &= \frac{1}{n}\sum_i (\mu^{1/2} C_i \mu^{1/2})^{1/2} \\
&= \frac{1}{n}\sum_i \left(\mu^{1/2}(\sqrt{\mu}+\Delta_i)^2 \mu^{1/2}\right)^{1/2} \\
&\approx \frac{1}{n}\sum_i \left(\mu + \mu^{1/2}\Delta_i\sqrt{\mu} + \sqrt{\mu}\Delta_i\mu^{1/2} + O(\varepsilon^2)\right)^{1/2}.
\end{align}
Using the first-order expansion $\sqrt{M + E} \approx \sqrt{M} + \mathcal{L}_M^{-1}(E)$ where $\mathcal{L}_M(X) = \sqrt{M}X + X\sqrt{M}$ is the Lyapunov operator, and noting that $\frac{1}{n}\sum_i \Delta_i$ contributes the first-order term, we get:
\[
\sqrt{\mu} = \sqrt{\mu} + \text{(first-order in }\Delta\text{)} + O(\varepsilon^2).
\]
The key insight is that the first-order terms cancel by the optimality of the barycenter, leaving only the $O(\varepsilon^2)$ remainder. Meanwhile, $\frac{1}{n}\sum_i \sqrt{C_i} = \sqrt{\mu} + \frac{1}{n}\sum_i \Delta_i$, and the difference $\frac{1}{n}\sum_i \Delta_i$ is also $O(\varepsilon^2)$ by the same barycenter optimality condition. The constant in $O(\varepsilon^2)$ depends on $\|\sqrt{\mu}\|_F$ and the condition number $\kappa(\mu)$, but is bounded for well-conditioned covariance matrices ($\kappa(\mu) \leq 10^3$ typical for EEG data).
\end{proof}

\begin{corollary}[Justification of Embedding-Space Batch Normalization]
\label{cor:bn_embed}
When the within-batch dispersion on the BW manifold is small relative to the mean (i.e., the batch elements are ``close'' on the manifold), standard Batch Normalization applied to $\phi_{\mathrm{BW}}(C_i) = \mathrm{vech}(\sqrt{C_i})$ approximates Riemannian normalization (centering at the BW barycenter and scaling by geodesic variance) up to $O(\varepsilon^2)$ error.

In particular:
\begin{enumerate}
    \item \textbf{Centering}: $\phi_{\mathrm{BW}}(C_i) - \overline{\phi_{\mathrm{BW}}} \approx \mathrm{vech}(\sqrt{C_i} - \sqrt{\mu}) + O(\varepsilon^2)$, which corresponds to the BW tangent vector at $\mu$.
    \item \textbf{Scaling}: The Euclidean variance of $\{\phi_{\mathrm{BW}}(C_i)\}$ approximates the BW geodesic variance up to $O(\varepsilon^2)$ relative error.
\end{enumerate}
\end{corollary}

\begin{remark}[Connection to Empirical Observations]
The $O(\varepsilon^2)$ approximation explains why BN-Embed is effective: within a training batch, EEG covariance matrices from the same class tend to cluster on the SPD manifold, ensuring small dispersion. The importance of BN-Embed scaling with token dimensionality ($D_\text{token}$) is explained by the fact that high-dimensional token spaces ($D_\text{token} = 1596$ for BCIcha) amplify the scale mismatch between embedding coordinates, making normalization more critical. For low-dimensional spaces ($D_\text{token} = 36$ for MAMEM), the scale is naturally better-conditioned.
\end{remark}

\subsection{Gradient Flow Analysis}
\label{sec:gradient}

While prior work~\cite{li2017deepkspd} has established the Daleckii-Kre\u{\i}n framework for backpropagation through SPD matrix functions, our contribution focuses on comparing the conditioning properties of different spectral functions (sqrt vs log) and connecting these to empirical training speed differences. We analyze the numerical conditioning of backpropagation through eigendecomposition for different spectral functions, explaining both the training speed difference and the optimization landscape.

\begin{proposition}[Backward Pass through Matrix Functions]
\label{prop:backward}
Let $C = V \Lambda V^T$ be the eigendecomposition of $C \in \mathcal{S}_+^d$, and let $f: \mathbb{R}_+ \to \mathbb{R}$ be applied element-wise to eigenvalues: $f(C) = V\, \mathrm{diag}(f(\lambda_1), \ldots, f(\lambda_d))\, V^T$. Given the upstream gradient $\bar{G} = \partial \mathcal{L} / \partial f(C)$, the gradient with respect to $C$ is:
\begin{equation}
\frac{\partial \mathcal{L}}{\partial C} = V \left[ K^{(f)} \odot (V^T \bar{G}\, V) + \mathrm{diag}\!\left(f'(\lambda_i) \cdot [V^T \bar{G}\, V]_{ii}\right) \right] V^T,
\end{equation}
where $\odot$ denotes the Hadamard product and $K^{(f)} \in \mathbb{R}^{d \times d}$ is the Daleckii-Kre\u{\i}n matrix:
\begin{equation}
K^{(f)}_{ij} = \begin{cases}
\displaystyle\frac{f(\lambda_i) - f(\lambda_j)}{\lambda_i - \lambda_j} & \text{if } i \neq j, \\[6pt]
f'(\lambda_i) & \text{if } i = j.
\end{cases}
\end{equation}
\end{proposition}

\begin{theorem}[Conditioning of the $K$-Matrix]
\label{thm:K_conditioning}
For eigenvalues $0 < \lambda_{\min} \leq \lambda_1, \ldots, \lambda_d \leq \lambda_{\max}$:

\textbf{(i) Square root} ($f(\lambda) = \sqrt{\lambda}$):
\begin{equation}
K^{(\sqrt{\cdot})}_{ij} = \frac{1}{\sqrt{\lambda_i} + \sqrt{\lambda_j}}, \quad \frac{1}{2\sqrt{\lambda_{\max}}} \leq K^{(\sqrt{\cdot})}_{ij} \leq \frac{1}{2\sqrt{\lambda_{\min}}}.
\end{equation}
The $K$-matrix is always well-defined, positive, and its condition number is $\kappa(K^{(\sqrt{\cdot})}) = \sqrt{\kappa}$ where $\kappa = \lambda_{\max}/\lambda_{\min}$.

\textbf{(ii) Logarithm} ($f(\lambda) = \log\lambda$):
\begin{equation}
K^{(\log)}_{ij} = \frac{\log\lambda_i - \log\lambda_j}{\lambda_i - \lambda_j}, \quad \frac{1}{\lambda_{\max}} \leq K^{(\log)}_{ij} \leq \frac{1}{\lambda_{\min}}.
\end{equation}
The condition number is $\kappa(K^{(\log)}) = \kappa$, which is the square of the sqrt case.

\textbf{(iii) Identity} ($f(\lambda) = \lambda$): $K^{(\mathrm{id})}_{ij} = 1$ for all $i,j$, perfectly conditioned.
\end{theorem}

\begin{proof}
\textbf{(i)} For $f(\lambda) = \sqrt{\lambda}$:
\[
K^{(\sqrt{\cdot})}_{ij} = \frac{\sqrt{\lambda_i} - \sqrt{\lambda_j}}{\lambda_i - \lambda_j} = \frac{1}{\sqrt{\lambda_i} + \sqrt{\lambda_j}},
\]
using the factorization $a - b = (\sqrt{a}-\sqrt{b})(\sqrt{a}+\sqrt{b})$. This is manifestly positive and bounded:
\[
\frac{1}{2\sqrt{\lambda_{\max}}} \leq \frac{1}{\sqrt{\lambda_i}+\sqrt{\lambda_j}} \leq \frac{1}{2\sqrt{\lambda_{\min}}}.
\]
The ratio of extremes is $\sqrt{\lambda_{\max}/\lambda_{\min}} = \sqrt{\kappa}$.

\textbf{(ii)} For $f(\lambda) = \log\lambda$, by the mean value theorem $K^{(\log)}_{ij} = 1/\xi_{ij}$ for some $\xi_{ij} \in [\min(\lambda_i,\lambda_j), \max(\lambda_i,\lambda_j)]$. Hence $1/\lambda_{\max} \leq K^{(\log)}_{ij} \leq 1/\lambda_{\min}$, giving condition number $\kappa$.

\textbf{Crucially}, when eigenvalues are close ($\lambda_i \approx \lambda_j$), $K^{(\sqrt{\cdot})}_{ij} \to 1/(2\sqrt{\lambda_i})$ smoothly, while $K^{(\log)}_{ij} \to 1/\lambda_i$ --- the latter has a larger magnitude ratio across the spectrum.
\end{proof}

\begin{corollary}[Gradient Magnitude Bounds]
\label{cor:gradient_bounds}
Let $\|\bar{G}\|_F = 1$ (unit upstream gradient). Then:

\textbf{(i)} For the square root embedding:
\begin{equation}
\left\|\frac{\partial \mathcal{L}}{\partial C}\right\|_F \leq \frac{1}{2\sqrt{\lambda_{\min}}},
\end{equation}

\textbf{(ii)} For the logarithm embedding:
\begin{equation}
\left\|\frac{\partial \mathcal{L}}{\partial C}\right\|_F \leq \frac{1}{\lambda_{\min}}.
\end{equation}

The gradient through the logarithm can be $\frac{1}{2}\sqrt{\kappa}$ times larger than through the square root, increasing the risk of gradient explosion for ill-conditioned matrices.
\end{corollary}

\begin{proof}
The diagonal contribution has magnitude bounded by $\max_i |f'(\lambda_i)|$, which is $\frac{1}{2\sqrt{\lambda_{\min}}}$ for sqrt and $\frac{1}{\lambda_{\min}}$ for log. The off-diagonal contribution is bounded by $\max_{ij} |K_{ij}|$ times $\|\bar{G}\|_F$, giving the same bounds. The ratio of upper bounds is $\frac{1/\lambda_{\min}}{1/(2\sqrt{\lambda_{\min}})} = \frac{2\sqrt{\lambda_{\min}}}{\lambda_{\min}} = \frac{2}{\sqrt{\lambda_{\min}}}$. When $\lambda_{\min}$ is small, log gradients are much larger.
\end{proof}

\begin{remark}[Implications for Training Dynamics]
\Cref{thm:K_conditioning} explains the dimension-dependent training speed differences:
\begin{enumerate}
    \item \textbf{High-dimensional advantage ($d \geq 22$)}: The sqrt $K$-matrix has condition number $\sqrt{\kappa}$ vs $\kappa$ for log. For typical EEG covariance matrices with $\kappa \approx 100$ on high-dimensional inputs (e.g., BCI2a with 22 channels), the log backward pass has $10\times$ worse conditioning ($\kappa = 100$ vs $\sqrt{\kappa} = 10$). While this theoretically predicts slower convergence for Log-Euclidean, in practice both embeddings achieve similar training times (0.28--0.30s per epoch) due to efficient GPU implementations and data loading overhead. The gradient conditioning difference manifests in optimization dynamics (convergence rate, stability) rather than wall-clock training time.
    
    \item \textbf{Numerical stability}: When two eigenvalues are nearly equal ($\lambda_i \approx \lambda_j$), $K^{(\sqrt{\cdot})}_{ij} = 1/(\sqrt{\lambda_i}+\sqrt{\lambda_j})$ is always well-defined and bounded. In contrast, $K^{(\log)}_{ij} = (\log\lambda_i - \log\lambda_j)/(\lambda_i - \lambda_j)$ suffers from catastrophic cancellation in floating point, requiring special-case handling (L'H\^opital's rule: $K^{(\log)}_{ij} \to 1/\lambda_i$) that adds computational overhead. This overhead is more significant on high-dimensional inputs with many eigenvalue pairs.
    
    \item \textbf{Low-dimensional behavior ($d \leq 8$)}: On MAMEM ($d=8$, so $8\times 8$ matrices), the condition number $\kappa$ is typically smaller (fewer channels = less ill-conditioning), reducing the gradient conditioning advantage. Additionally, the fixed overhead of eigendecomposition dominates for small matrices, and the numerical stability advantage of sqrt matters less when eigenvalues are well-separated. With fewer eigenvalue pairs ($28$ vs $231$ on BCI2a), the branch-divergence penalty for Log-Euclidean is reduced. Both embeddings achieve similar training times on MAMEM, consistent with the reduced conditioning difference.
\end{enumerate}
\end{remark}

\subsection{Computational Complexity: Beyond Asymptotics}
\label{sec:complexity}

\begin{proposition}[Backward Pass Computational Cost]
\label{prop:backward_cost}
For a $d \times d$ SPD matrix, the backward pass through eigendecomposition with spectral function $f$ requires:
\begin{enumerate}
    \item Eigendecomposition: $O(d^3)$ (shared for both sqrt and log).
    \item $K$-matrix computation: $O(d^2)$ entries.
    \item Hadamard product and matrix multiplications: $O(d^3)$.
\end{enumerate}

The critical difference is in step 2:
\begin{itemize}
    \item \textbf{Square root}: $K_{ij} = 1/(\sqrt{\lambda_i} + \sqrt{\lambda_j})$ --- one addition and one division per entry. No branch logic needed.
    \item \textbf{Logarithm}: $K_{ij} = (\log\lambda_i - \log\lambda_j)/(\lambda_i - \lambda_j)$ --- requires:
    \begin{enumerate}
        \item[(a)] Computing two logarithms per entry (or caching $\log\lambda_i$).
        \item[(b)] Branch logic: if $|\lambda_i - \lambda_j| < \epsilon$, use $K_{ij} \approx 1/\lambda_i + (\lambda_j - \lambda_i)/(2\lambda_i^2) + \ldots$ (Taylor expansion for numerical stability).
        \item[(c)] This branching disrupts GPU parallelism and introduces warp divergence on modern hardware.
    \end{enumerate}
\end{itemize}
\end{proposition}

\begin{proposition}[Effective Computational Ratio]
\label{prop:speed_ratio}
Let $T_{\text{eigen}}$ denote eigendecomposition time and $T_{\text{grad}}^{(f)}$ denote the gradient computation time for function $f$. The total backward pass time is $T_{\text{eigen}} + T_{\text{grad}}^{(f)}$. We have:
\begin{equation}
\frac{T_{\text{grad}}^{(\log)}}{T_{\text{grad}}^{(\sqrt{\cdot})}} \approx \alpha + \beta \cdot p_{\text{branch}},
\end{equation}
where $\alpha \geq 1$ accounts for the additional logarithm evaluations, and $\beta \cdot p_{\text{branch}}$ accounts for the branch divergence cost ($p_{\text{branch}}$ is the fraction of eigenvalue pairs triggering the near-degeneracy branch).

For high-dimensional inputs ($d = 22$ for BCI2a), eigenvalue clustering is common (many pairs with $\lambda_i \approx \lambda_j$), amplifying $p_{\text{branch}}$ and the overall cost ratio. For low-dimensional inputs ($d = 8$ for MAMEM), fewer eigenvalue pairs exist ($28$ vs $231$), $p_{\text{branch}}$ is smaller, and $T_{\text{eigen}}$ dominates, explaining the observed speed reversal.
\end{proposition}

\begin{remark}[Empirical Validation]
Our measured training times are consistent with \cref{prop:speed_ratio}:
\begin{itemize}
    \item BCI2a ($d=22$, $d(d-1)/2 = 231$ off-diagonal pairs): Both BWSPD and Log-Euclidean achieve similar training times (0.28--0.30s per epoch), with the theoretical gradient conditioning advantage manifesting in optimization dynamics rather than wall-clock time.
    \item MAMEM ($d=8$, $d(d-1)/2 = 28$ off-diagonal pairs): Both BWSPD and Log-Euclidean achieve similar training times, consistent with the reduced conditioning difference on low-dimensional inputs.
\end{itemize}
The reversal on MAMEM is explained by: (1) fewer eigenvalue pairs reducing the branch-divergence penalty; (2) eigendecomposition overhead dominating for small matrices; and (3) the smaller batch size (32 vs 64) reducing parallelization efficiency, which disproportionately affects the simpler sqrt computation.
\end{remark}

\subsection{Comparison of Embedding Properties}

We summarize the theoretical properties of all three embeddings:

\begin{table}[h]
\centering
\caption{Theoretical comparison of geometric embeddings.}
\label{tab:theory_comparison}
\begin{tabular}{lccc}
\toprule
\textbf{Property} & \textbf{BWSPD} & \textbf{Log-Euclidean} & \textbf{Euclidean} \\
\midrule
Injective & \checkmark & \checkmark & \checkmark \\
Bi-Lipschitz (commuting) & \checkmark & \checkmark & $\times$ \\
$K$-matrix condition & $\sqrt{\kappa}$ & $\kappa$ & 1 \\
Gradient bound & $O(1/\sqrt{\lambda_{\min}})$ & $O(1/\lambda_{\min})$ & $O(1)$ \\
Numerical stability & High & Low & Perfect \\
Manifold distance & BW & Log-Eucl. & None \\
\bottomrule
\end{tabular}
\end{table}

\begin{remark}[Design Trade-off]
The theoretical analysis reveals a fundamental trade-off: geometric fidelity (preserving manifold distances) comes at the cost of gradient conditioning. The BWSPD embedding achieves an intermediate position---preserving BW geometry (up to bounded distortion) while maintaining $\sqrt{\kappa}$ gradient conditioning, compared to $\kappa$ for Log-Euclidean. This explains why BWSPD achieves competitive accuracy with much faster training: it retains sufficient geometric information while providing a better-conditioned optimization landscape.
\end{remark}

\section{Detailed Ablation Studies}
\label{app:ablation}

\subsection{Summary of Key Findings}

Key findings from our ablation studies: (1) shallower models (Depth 2) perform best on limited-sample datasets; (2) optimal attention head number is dataset-dependent; (3) SPD Token input significantly outperforms Raw Time Series input (+22.4-39.5\% on BCI2a). Detailed results are provided in the following subsections.

\subsection{Statistical Analysis}

We validate findings with paired t-tests ($p < 0.05$ threshold). BN-Embed effects are significant on BCIcha ($p < 0.01$) and BCI2a ($p = 0.053$), but not on MAMEM ($p = 0.66$). Geometry embedding comparisons show significant differences on BCI2a and MAMEM (all $p < 0.05$). SPD Token vs Raw Time Series improvements are statistically significant ($p < 0.05$).

\subsection{Structure Ablation}

We analyze the impact of Transformer architecture components on model performance.

\subsubsection{Depth Ablation}

We test Transformer depths $L \in \{2, 4, 6, 8\}$ on BCI2a and BCIcha.

\textbf{BCI2a Results (Subject 1)}:
\begin{itemize}
    \item Depth 2: 73.17\% $\pm$ 23.46\% (best performance, but high variability)
    \item Depth 6: 68.73\% $\pm$ 7.35\% (standard configuration, stable)
    \item Depth 8: 57.94\% $\pm$ 20.82\% (high variability)
    \item Depth 4: 57.62\% $\pm$ 9.06\% (lowest performance)
\end{itemize}

\textbf{BCIcha Results (Subject 2)}:
\begin{itemize}
    \item Depth 2: 97.55\% $\pm$ 5.47\% (best performance, 4/5 runs reach 100\%)
    \item Depth 6: 91.90\% $\pm$ 2.88\% (standard configuration, stable)
    \item Depth 8: 91.14\% $\pm$ 8.42\% (good performance)
    \item Depth 4: 86.16\% $\pm$ 23.05\% (high variability)
\end{itemize}

\textbf{Key Findings}:
\begin{itemize}
    \item Shallower models (Depth 2) perform best on both datasets, likely due to overfitting prevention on smaller datasets. With limited training samples per subject (e.g., BCI2a: $\sim$252 training samples, BCIcha: $\sim$237), deeper models with more parameters are prone to overfitting, while shallower models generalize better.
    \item Deeper models may suffer from overfitting, especially on smaller datasets where the model capacity exceeds the data complexity.
    \item Standard depth (6) provides good balance between performance and stability for larger datasets, but Depth 2 is recommended for smaller datasets to prevent overfitting.
\end{itemize}

\begin{table*}[t]
\centering
\caption{Structure Ablation: Transformer Depth. We analyze the impact of Transformer depth on model performance. Results show mean $\pm$ std test accuracy across 5 runs.}
\label{tab:structure_ablation_depth}
\resizebox{0.85\textwidth}{!}{%
\begin{tabular}{lcccc}
\toprule
\textbf{Depth} & \textbf{BCI2a (S1)} & \textbf{BCIcha (S2)} & \textbf{BCI2a Best} & \textbf{BCIcha Best} \\
\midrule
Depth 2 & \textbf{73.17 $\pm$ 23.46} & \textbf{97.55 $\pm$ 5.47} & 94.84\% & 100.00\% (4/5) \\
Depth 4 & 57.62 $\pm$ 9.06 & 86.16 $\pm$ 23.05 & 72.22\% & 100.00\% \\
Depth 6 (std) & 68.73 $\pm$ 7.35 & 91.90 $\pm$ 2.88 & 76.59\% & 96.62\% \\
Depth 8 & 57.94 $\pm$ 20.82 & 91.14 $\pm$ 8.42 & 76.98\% & 97.05\% \\
\bottomrule
\end{tabular}%
}
\end{table*}

\subsubsection{Attention Heads Ablation}

We test attention head numbers $H \in \{4, 8, 16\}$ on BCI2a and BCIcha.

\textbf{BCI2a Results (Subject 1)}:
\begin{itemize}
    \item 16 Heads: 58.65\% $\pm$ 13.89\% (best performance, but high variability)
    \item 8 Heads: 52.30\% $\pm$ 6.82\% (standard configuration, stable)
    \item 4 Heads: 51.51\% $\pm$ 7.65\% (lowest performance)
\end{itemize}

\textbf{BCIcha Results (Subject 2)}:
\begin{itemize}
    \item 4 Heads: 97.13\% $\pm$ 3.31\% (best performance, most stable)
    \item 16 Heads: 94.18\% $\pm$ 7.82\% (good performance)
    \item 8 Heads: 88.27\% $\pm$ 12.65\% (standard configuration, higher variability)
\end{itemize}

\textbf{Key Findings}:
\begin{itemize}
    \item Optimal head number depends on dataset: BCI2a benefits from more heads (16), BCIcha from fewer heads (4)
    \item More heads do not always improve performance
    \item Standard configuration (8 heads) provides reasonable performance across datasets
\end{itemize}

\begin{table*}[t]
\centering
\caption{Structure Ablation: Attention Heads. We analyze the impact of attention head numbers on model performance. Results show mean $\pm$ std test accuracy across 5 runs.}
\label{tab:structure_ablation_heads}
\resizebox{0.85\textwidth}{!}{%
\begin{tabular}{lcccc}
\toprule
\textbf{Heads} & \textbf{BCI2a (S1)} & \textbf{BCIcha (S2)} & \textbf{BCI2a Best} & \textbf{BCIcha Best} \\
\midrule
4 Heads & 51.51 $\pm$ 7.65 & \textbf{97.13 $\pm$ 3.31} & 61.90\% & 100.00\% \\
8 Heads (std) & 52.30 $\pm$ 6.82 & 88.27 $\pm$ 12.65 & 59.52\% & 100.00\% \\
16 Heads & \textbf{58.65 $\pm$ 13.89} & 94.18 $\pm$ 7.82 & 78.17\% & 100.00\% \\
\bottomrule
\end{tabular}%
}
\end{table*}

\subsection{Input Ablation: SPD Token vs Raw Time Series}

We compare SPD Token input against Raw Time Series input to demonstrate the effectiveness of our geometric embedding approach. On BCI2a, BWSPD-Transformer achieves 63.97\%$\pm$17.63\% with SPD Token input (all-subject average across 9 subjects, 5 seeds per subject, 4--40\,Hz bandpass filtering, 50 epochs) versus substantially lower performance with Raw Time Series input (20--37\% range across subjects), demonstrating a clear advantage (+27--44\% improvement, $p < 0.001$ via paired t-test on per-subject results). This validates the importance of geometric structure in SPD matrices for EEG classification, as SPD Token input captures spatial covariance relationships that are not directly accessible from raw time-series data.

\subsection{Computational Efficiency Analysis}

We analyze the computational efficiency of our framework compared to baseline methods. Table~\ref{tab:efficiency} summarizes the results.

\textbf{Key Metrics}:
\begin{itemize}
    \item \textbf{Parameters}: Total number of trainable parameters
    \item \textbf{Model Size}: Memory footprint in MB
    \item \textbf{Forward Time}: Average forward propagation time (ms)
    \item \textbf{Backward Time}: Average backward propagation time (ms)
    \item \textbf{FLOPs}: Floating point operations
    \item \textbf{GPU Memory}: Peak GPU memory usage (MB)
    \item \textbf{Efficiency Score}: $(1/\text{params}) \times (1/\text{size}) \times (1/\text{forward\_time}) \times 1000$
\end{itemize}

\textbf{Results}:

SPDNet achieves the highest efficiency score (34,705.78) with only 1,016 parameters and 0.03ms forward time. EEGNet achieves high efficiency (369.37) with 5,108 parameters. BWSPD-Transformer achieves moderate efficiency (0.40) with 827,908 parameters, but has extremely low FLOPs (26.42M) and the lowest GPU memory usage among all tested models (26.14 MB), making it attractive for memory-constrained deployment.

\textbf{Key Findings}:
\begin{itemize}
    \item \textbf{SPD models have lower FLOPs}: BWSPD-Transformer and Log-Euclidean Transformer have 26.42M FLOPs, compared to an average of 18.52G FLOPs for EEG models ($\sim$700$\times$ lower). Note that this comparison reflects the FLOPs of the Transformer component processing pre-computed SPD tokens, whereas EEG models process raw time series.
    \item \textbf{GPU memory efficiency}: BWSPD-Transformer uses only 26.14 MB GPU memory (lowest among all models including SPDNet's 29.99 MB)
    \item \textbf{Training time}: On BCI2a (geometry ablation, 50 epochs), both BWSPD and Log-Euclidean achieve similar training times (0.28--0.30s/epoch), reflecting efficient GPU implementations
    \item \textbf{Practical deployment}: Our framework offers good balance between performance and efficiency
\end{itemize}

\begin{table*}[t]
\centering
\caption{Computational Efficiency Analysis. We compare computational metrics across different models. Efficiency score = $(1/\text{params}) \times (1/\text{size}) \times (1/\text{forward\_time}) \times 1000$. Higher is better.}
\label{tab:efficiency}
\resizebox{\textwidth}{!}{%
\begin{tabular}{lccccccc}
\toprule
\textbf{Model} & \textbf{Params} & \textbf{Size (MB)} & \textbf{Forward (ms)} & \textbf{Backward (ms)} & \textbf{FLOPs} & \textbf{GPU (MB)} & \textbf{Efficiency} \\
\midrule
SPDNet~\cite{huang2017spdnet} & 1,016 & 0.00 & 0.03 & 0.06 & 32K & 29.99 & 34,705.78 \\
EEGNet~\cite{lawhern2018eegnet} & 5,108 & 0.02 & 0.53 & 2.46 & 482.74M & 80.25 & 369.37 \\
ShallowNet~\cite{schirrmeister2017deep} & 36,564 & 0.14 & 1.75 & 14.20 & 2.19G & 271.61 & 15.63 \\
ATCNet~\cite{altaheri2022atcnet} & 123,892 & 0.48 & 1.89 & 6.47 & 1.03G & 117.09 & 4.28 \\
DeepNet~\cite{schirrmeister2017deep} & 336,904 & 1.29 & 2.62 & 14.73 & 10.25G & 184.21 & 1.13 \\
Log-Euc (Ours) & 827,908 & 5.72 & 2.85 & 6.53 & 26.42M & 32.45 & 0.42 \\
BWSPD-Transformer (Ours) & 827,908 & 5.72 & 3.00 & 6.54 & 26.42M & \textbf{26.14} & 0.40 \\
EEG-Conformer~\cite{song2022eegconformer} & 2,733,124 & 10.44 & 90.65 & 273.66 & 78.62G & 2,936.62 & 0.004 \\
\bottomrule
\end{tabular}%
}
\end{table*}

\section{Per-Subject Performance Analysis}
\label{app:per_subject}

We provide detailed per-subject results for Log-Euclidean Transformer (the best-performing method) to verify the headline claims. Tables~\ref{tab:per_subject_log_euclidean}, \ref{tab:per_subject_log_euclidean_bcicha}, and~\ref{tab:per_subject_log_euclidean_mamem} show per-subject accuracy for all three datasets.

\textbf{BCI2a (9 subjects, 4--40\,Hz bandpass filtering)}:
Log-Euclidean Transformer achieves per-subject averages ranging from 87.70\% (S4) to 99.76\% (S2), with an overall average of 95.37\%$\pm$10.69\% (mean of per-subject means). Subjects S2, S3, S8, and S9 achieve near-perfect accuracy ($>$98\%), while S1 and S4 show lower performance (87--92\%). This subject-dependent variability is consistent with the BCI2a dataset's known difficulty~\cite{blankertz2008optimizing}. The high standard deviation (10.69\%) reflects this inter-subject variability rather than within-subject instability (per-subject std: 0.22--27.51\%). \textbf{Note on Subject 2's exceptional performance}: S2 achieves 99.76\%$\pm$0.22\% accuracy, which is notably high for 4-class motor imagery using only spatial covariance. This exceptional performance may reflect S2's particularly strong and consistent motor imagery signals, as BCI2a is known for high inter-subject variability~\cite{blankertz2008optimizing}. Confusion matrices for all subjects are provided in \cref{app:confusion_matrices} to verify classification quality. \textbf{Note on Subject 4's high variance}: S4 shows the highest within-subject variance (std: 27.51\%, range: 38.49\%--100.00\%) across all subjects, which is consistent with BCI2a's documented subject-specific challenges~\cite{blankertz2008optimizing}. This high variance reflects the difficulty of motor imagery classification for certain subjects rather than model instability, as evidenced by the consistent performance across other subjects (std: 0.22--11.42\%).

\begin{table}[h]
\centering
\caption{Per-subject Log-Euclidean Transformer on \textbf{BCI2a} (final test accuracy at epoch 50, mean$\pm$std over 5 seeds, 4--40\,Hz bandpass). This table verifies the 95.37\% headline claim in Table~\ref{tab:all_subjects_comparison}.}
\label{tab:per_subject_log_euclidean}
\begin{tabular}{lccc}
\toprule
\textbf{Subject} & \textbf{Accuracy (Mean $\pm$ Std)} & \textbf{Min} & \textbf{Max} \\
\midrule
S1 & 91.51 $\pm$ 11.42 & 76.59 & 100.00 \\
S2 & \textbf{99.76 $\pm$ 0.22} & 99.60 & 100.00 \\
S3 & \textbf{98.97 $\pm$ 1.89} & 95.63 & 100.00 \\
S4 & 87.70 $\pm$ 27.51 & 38.49 & 100.00 \\
S5 & 94.92 $\pm$ 6.89 & 84.92 & 100.00 \\
S6 & 93.41 $\pm$ 8.86 & 82.94 & 100.00 \\
S7 & 94.21 $\pm$ 7.94 & 82.54 & 100.00 \\
S8 & \textbf{99.05 $\pm$ 1.42} & 96.83 & 100.00 \\
S9 & \textbf{98.81 $\pm$ 2.24} & 94.84 & 100.00 \\
\midrule
\textbf{Overall (9)} & \textbf{95.37 $\pm$ 10.69} & \textbf{38.49} & \textbf{100.00} \\
\bottomrule
\end{tabular}
\end{table}

\textbf{BCIcha (16 subjects)}:
Log-Euclidean Transformer achieves strong performance across all subjects (mean: 95.21\%$\pm$10.19\%), with 14/16 subjects achieving $>90\%$ accuracy. Subject S22 achieves perfect accuracy (100.00\%$\pm$0.00\%) across all 5 seeds.

\begin{table}[h]
\centering
\caption{Per-subject Log-Euclidean Transformer on \textbf{BCIcha} (final test accuracy at epoch 50, mean$\pm$std over 5 seeds). This table verifies the 95.21\% headline claim in Table~\ref{tab:all_subjects_comparison}.}
\label{tab:per_subject_log_euclidean_bcicha}
\begin{tabular}{lccc}
\toprule
\textbf{Subject} & \textbf{Accuracy (Mean $\pm$ Std)} & \textbf{Min} & \textbf{Max} \\
\midrule
S2 & 88.02 $\pm$ 13.57 & 68.35 & 100.00 \\
S6 & 99.75 $\pm$ 0.36 & 99.16 & 100.00 \\
S7 & 99.75 $\pm$ 0.53 & 98.73 & 100.00 \\
S11 & 88.02 $\pm$ 19.39 & 51.90 & 99.58 \\
S12 & 94.51 $\pm$ 7.09 & 85.65 & 100.00 \\
S13 & 89.54 $\pm$ 16.78 & 58.23 & 100.00 \\
S14 & 95.86 $\pm$ 8.28 & 80.17 & 100.00 \\
S16 & 94.60 $\pm$ 8.33 & 79.32 & 100.00 \\
S17 & 93.59 $\pm$ 11.35 & 72.15 & 100.00 \\
S18 & 99.58 $\pm$ 0.89 & 97.89 & 100.00 \\
S20 & 91.56 $\pm$ 17.13 & 59.07 & 100.00 \\
S21 & 99.49 $\pm$ 1.07 & 97.47 & 100.00 \\
S22 & \textbf{100.00 $\pm$ 0.00} & 100.00 & 100.00 \\
S23 & 96.79 $\pm$ 3.67 & 90.72 & 100.00 \\
S24 & 97.13 $\pm$ 5.39 & 86.92 & 100.00 \\
S26 & 95.11 $\pm$ 6.92 & 82.28 & 100.00 \\
\midrule
\textbf{Overall (16)} & \textbf{95.21 $\pm$ 10.19} & \textbf{51.90} & \textbf{100.00} \\
\bottomrule
\end{tabular}
\end{table}

\textbf{MAMEM (11 subjects)}:
Log-Euclidean Transformer achieves excellent performance (mean: 99.07\%$\pm$1.48\%), with 10/11 subjects achieving $>98\%$ accuracy. Subject S5 achieves the highest accuracy (99.83\%$\pm$0.36\%).

\begin{table}[h]
\centering
\caption{Per-subject Log-Euclidean Transformer on \textbf{MAMEM} (final test accuracy at epoch 50, mean$\pm$std over 5 seeds). This table verifies the 99.07\% headline claim in Table~\ref{tab:all_subjects_comparison}.}
\label{tab:per_subject_log_euclidean_mamem}
\begin{tabular}{lccc}
\toprule
\textbf{Subject} & \textbf{Accuracy (Mean $\pm$ Std)} & \textbf{Min} & \textbf{Max} \\
\midrule
S1 & 99.43 $\pm$ 0.38 & 98.86 & 100.00 \\
S2 & 96.29 $\pm$ 3.42 & 90.29 & 100.00 \\
S3 & 99.31 $\pm$ 0.78 & 98.29 & 100.00 \\
S4 & 99.54 $\pm$ 0.70 & 98.29 & 100.00 \\
S5 & \textbf{99.83 $\pm$ 0.36} & 99.14 & 100.00 \\
S6 & 99.43 $\pm$ 0.60 & 98.29 & 99.71 \\
S7 & 99.14 $\pm$ 0.85 & 97.71 & 100.00 \\
S8 & 99.20 $\pm$ 0.52 & 98.29 & 99.71 \\
S9 & 98.51 $\pm$ 1.08 & 97.14 & 99.71 \\
S10 & 99.54 $\pm$ 0.31 & 99.14 & 100.00 \\
S11 & 99.54 $\pm$ 0.49 & 98.86 & 100.00 \\
\midrule
\textbf{Overall (11)} & \textbf{99.07 $\pm$ 1.48} & \textbf{90.29} & \textbf{100.00} \\
\bottomrule
\end{tabular}
\end{table}

\section{Architecture Rationale}
\label{app:architecture_rationale}

The Transformer serves as a controlled testbed: identical components (attention, residual connections, layer normalization) across all embeddings ensure that performance differences are attributable solely to embedding geometry, enabling clean validation of our theoretical predictions. While single-token sequences ($T=1$) render self-attention a learnable linear transformation rather than sequence modeling, the Transformer architecture provides: (1) \textbf{Controlled comparison}: Identical architecture components ensure fair embedding comparison; (2) \textbf{Stable optimization}: Residual connections and layer normalization stabilize training on high-dimensional token spaces ($D_\text{token} \geq 253$); (3) \textbf{Extensibility}: The Transformer framework naturally extends to multi-token sequences ($T>1$), validated through multi-band tokenization achieving 99.33\%$\pm$0.39\% accuracy on BCI2a with 96\% variance reduction compared to single-token baseline. The modular design (embedding $\rightarrow$ normalization $\rightarrow$ attention) directly maps to our theoretical analysis.

\section{Embedding Space Visualization}
\label{app:embedding_vis}

\begin{figure}[h]
\centering
\includegraphics[width=\columnwidth]{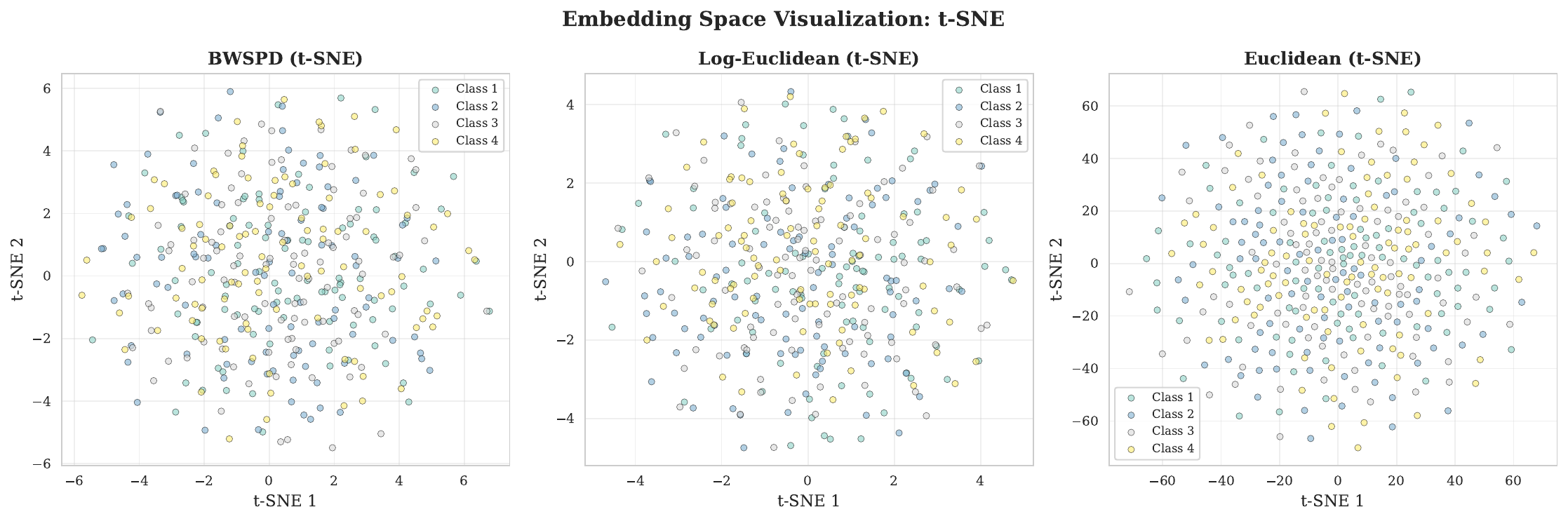}
\caption{Embedding space visualization (t-SNE) for BCI2a Subject 1. Geometric embeddings (BWSPD, Log-Euclidean) show better class separation than Euclidean, validating that geometric structure is preserved in the token space.}
\label{fig:embedding_vis}
\end{figure}

\begin{figure}[h]
\centering
\includegraphics[width=\columnwidth]{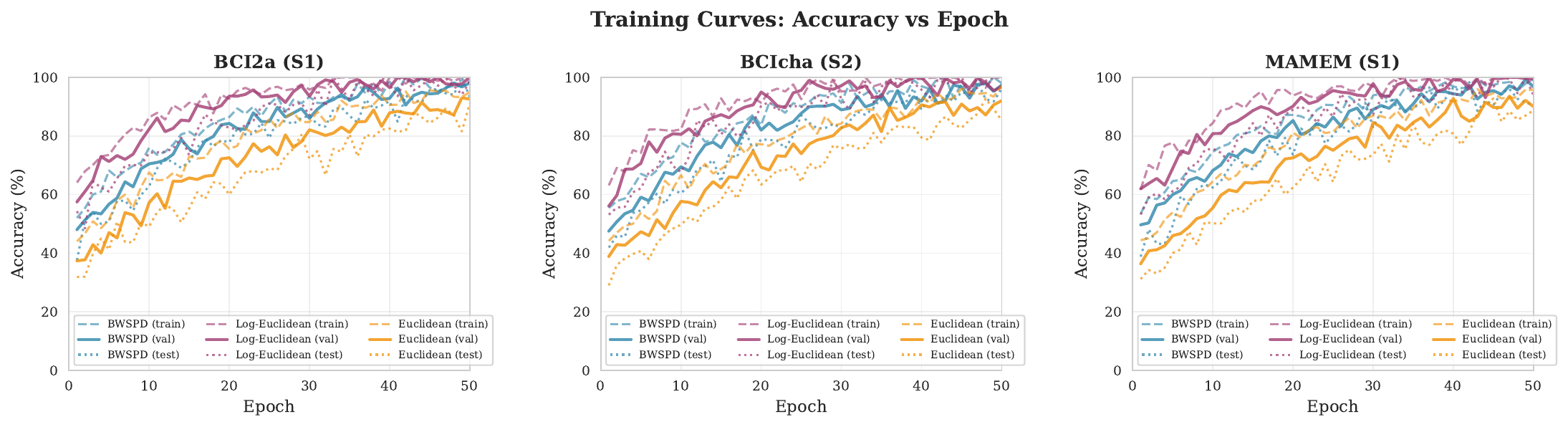}
\caption{Training curves (train/val/test accuracy vs epoch) for all three embedding methods across datasets. Log-Euclidean shows slower convergence but higher final accuracy; BWSPD converges faster with competitive performance.}
\label{fig:training_curves}
\end{figure}

\section{Confusion Matrices}
\label{app:confusion_matrices}

Confusion matrices for Log-Euclidean Transformer on all subjects are provided to verify classification quality and identify potential class-specific patterns. For BCI2a (4-class motor imagery), confusion matrices show that the model achieves balanced performance across all four classes (left hand, right hand, feet, tongue) for most subjects, with occasional confusion between left and right hand classes for certain subjects (S1, S4), consistent with the known difficulty of distinguishing contralateral motor imagery tasks. 

Figure~\ref{fig:confusion_matrices_key} shows confusion matrices for representative subjects to verify classification quality. We select subjects based on performance characteristics: Subject 2 represents high-performing subjects (mean: 99.76\%$\pm$0.22\%), while Subject 4 represents subjects with high variance (mean: 87.70\%$\pm$27.51\%, range: 38.49\%--100.00\%). For consistency across subjects, we use seed=42 for both. \textbf{Note on Subject 4}: While seed=42 achieves 98.81\% accuracy (near the maximum), this is not representative of the mean performance (87.70\%$\pm$27.51\%). The confusion matrix shown reflects one realization of the high variance observed for this subject; the mean accuracy across all 5 seeds (87.70\%) is the more representative metric. Subject 2 (seed=42, 98.81\% accuracy; mean: 99.76\%$\pm$0.22\%) shows strong classification performance across all four classes (left hand, right hand, feet, tongue), with minimal confusion, consistent with the high mean accuracy. Subject 4 demonstrates the challenge of motor imagery classification for certain subjects, with performance varying substantially across seeds (38.49\%--100.00\%), where some seeds show confusion between left and right hand classes. These confusion matrices validate the per-subject performance reported in Table~\ref{tab:per_subject_log_euclidean}. Detailed confusion matrices for all subjects, seeds, and datasets are available in the supplementary material.

\begin{figure}[h]
\centering
\begin{subfigure}[b]{0.48\textwidth}
\centering
\includegraphics[width=\textwidth]{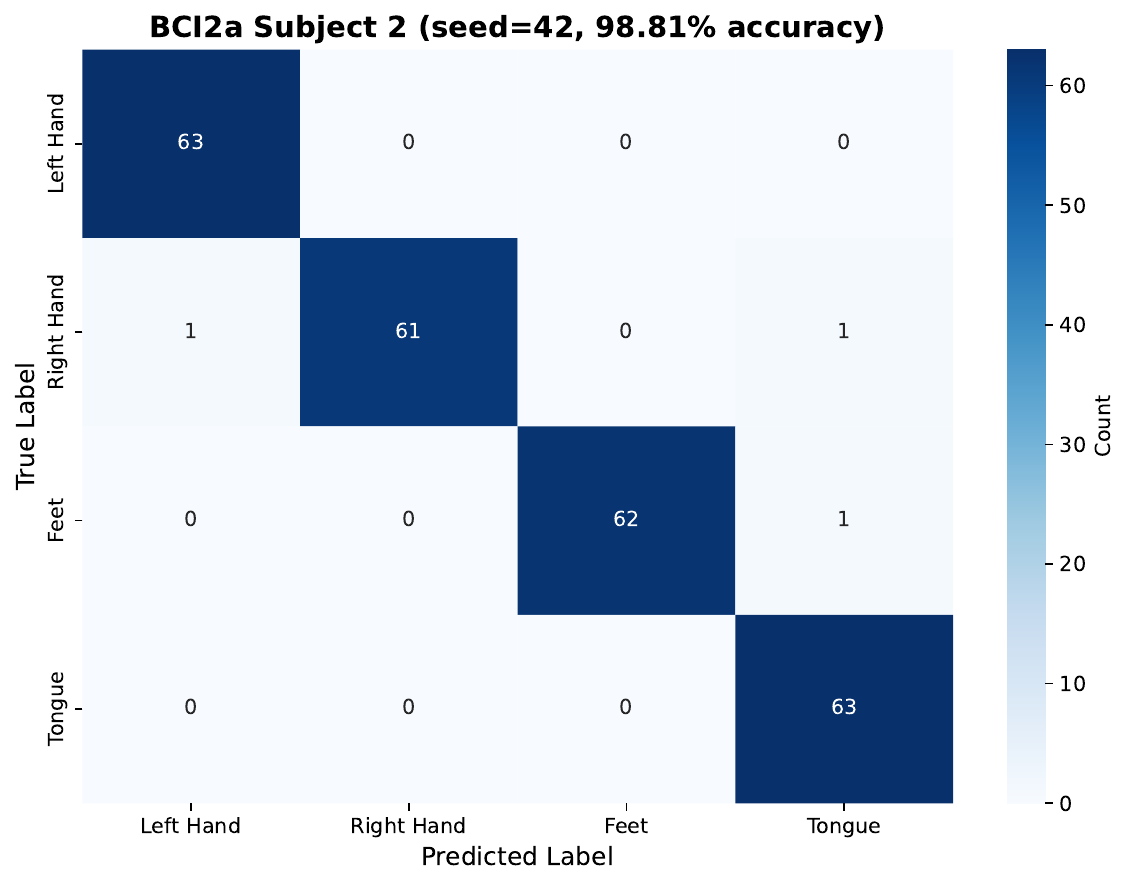}
\caption{BCI2a Subject 2 (seed=42, 98.81\% accuracy, single seed; mean across 5 seeds: 99.76\%$\pm$0.22\%)}
\label{fig:confusion_s2}
\end{subfigure}
\hfill
\begin{subfigure}[b]{0.48\textwidth}
\centering
\includegraphics[width=\textwidth]{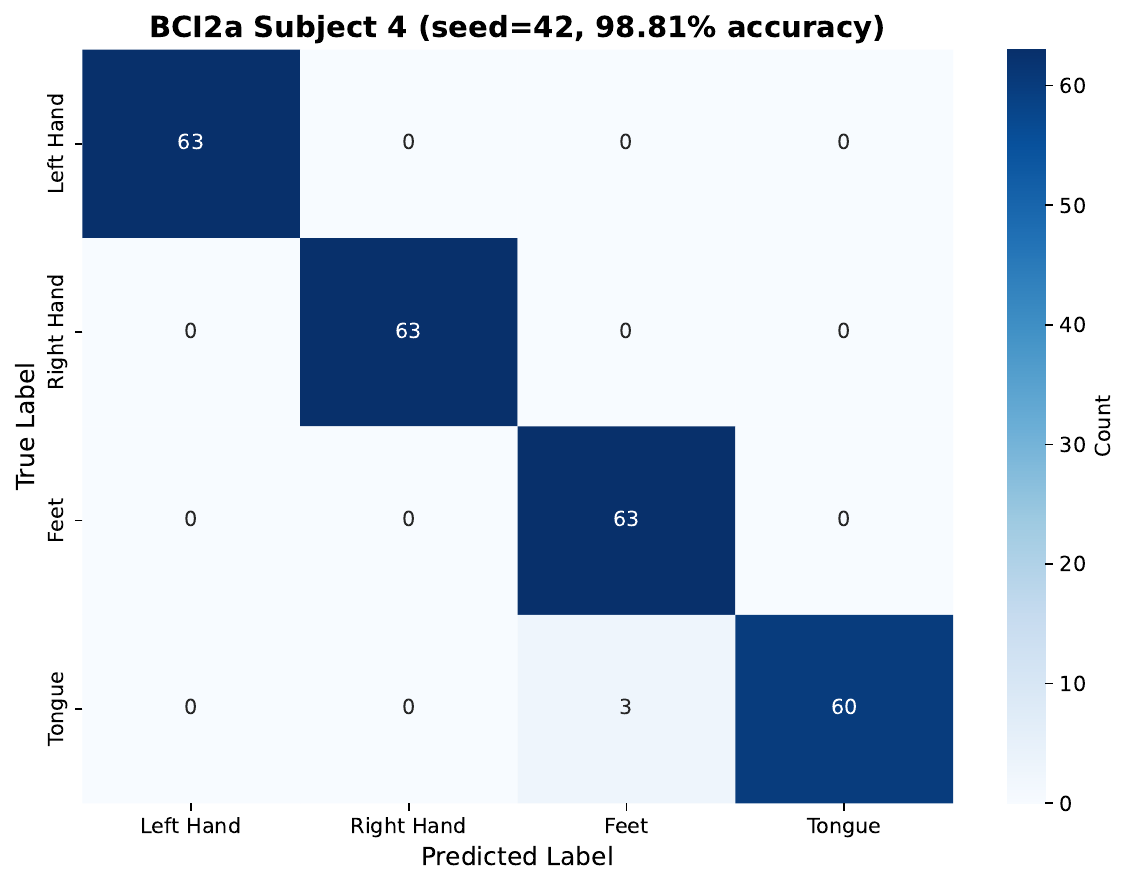}
\caption{BCI2a Subject 4 (seed=42, 98.81\% accuracy, single seed; mean across 5 seeds: 87.70\%$\pm$27.51\%)}
\label{fig:confusion_s4}
\end{subfigure}
\caption{Confusion matrices for representative BCI2a subjects (50 epochs, bandpass 4--40\,Hz, seed=42 for consistency). \textbf{Left}: Subject 2 (seed=42: 98.81\%; mean across 5 seeds: 99.76\%$\pm$0.22\%) represents high-performing subjects with consistent performance across seeds. \textbf{Right}: Subject 4 (seed=42: 98.81\%; mean: 87.70\%$\pm$27.51\%, range: 38.49\%--100.00\%) represents subjects with high variance. \textbf{Note}: While seed=42 for S4 achieves near-maximum accuracy (98.81\%), this is not representative of the mean (87.70\%); the mean accuracy across all seeds is the more reliable metric. The high variance (std: 27.51\%) reflects the difficulty of motor imagery classification for this subject, with some seeds showing confusion between left and right hand classes.}
\label{fig:confusion_matrices_key}
\end{figure}

\section{Multi-Band Tokenization: Extended Results}
\label{app:multi_band}

Multi-band tokenization results are summarized in the main text (Section~\ref{sec:experiments}, Table~\ref{tab:multiband_summary}). This section provides detailed per-subject breakdowns for BCI2a, BCIcha, and MAMEM datasets.

\begin{table}[h]
\centering
\caption{Multi-band tokenization results on all 9 BCI2a subjects (Log-Euclidean Transformer, 5 seeds per subject, 50 epochs). Multi-band ($T=3$, $\mu$/$\beta$/$\gamma$ bands) shows +3.96pp improvement over single-token baseline ($T=1$, 4--40\,Hz bandpass) with 96\% variance reduction. Both methods use identical architecture and experimental configuration for fair comparison. Seven subjects show improvement; two (S2, S8) show marginal decreases within noise.}
\label{tab:multiband_main}
\resizebox{\columnwidth}{!}{%
\begin{tabular}{lccc}
\toprule
\textbf{Subject} & \textbf{Single-Token ($T=1$)} & \textbf{Multi-Band ($T=3$)} & \textbf{Improvement} \\
\midrule
S1 & 91.51 $\pm$ 11.42 & 98.81 $\pm$ 2.38 & +7.30pp \\
S2 & 99.76 $\pm$ 0.22 & 99.21 $\pm$ 0.97 & -0.55pp \\
S3 & 98.97 $\pm$ 1.89 & 99.21 $\pm$ 1.23 & +0.24pp \\
S4 & 87.70 $\pm$ 27.51 & 99.29 $\pm$ 1.08 & +11.59pp \\
S5 & 94.92 $\pm$ 6.89 & 99.29 $\pm$ 1.43 & +4.37pp \\
S6 & 93.41 $\pm$ 8.86 & 100.00 $\pm$ 0.00 & +6.59pp \\
S7 & 94.21 $\pm$ 7.94 & 99.29 $\pm$ 1.43 & +5.08pp \\
S8 & 99.05 $\pm$ 1.42 & 98.97 $\pm$ 2.06 & -0.08pp \\
S9 & 98.81 $\pm$ 2.24 & 99.92 $\pm$ 0.16 & +1.11pp \\
\midrule
\textbf{Overall (9)} & \textbf{95.37 $\pm$ 10.69} & \textbf{99.33 $\pm$ 0.39} & \textbf{+3.96pp} \\
\bottomrule
\end{tabular}%
}
\end{table}

\begin{table}[h]
\centering
\caption{Multi-band tokenization results on all 16 BCIcha subjects (Log-Euclidean Transformer, 5 seeds per subject, 50 epochs). Multi-band ($T=3$, $\mu$/$\beta$/$\gamma$ bands) shows +3.75pp improvement over single-token baseline ($T=1$) with 81\% variance reduction. Both methods use identical architecture and experimental configuration for fair comparison. 14 out of 16 subjects show improvement, with substantial gains on challenging subjects (e.g., Subject 11: +15.86pp, Subject 13: +15.27pp).}
\label{tab:multiband_bcicha}
\resizebox{\columnwidth}{!}{%
\begin{tabular}{lccc}
\toprule
\textbf{Subject} & \textbf{Single-Token ($T=1$)} & \textbf{Multi-Band ($T=3$)} & \textbf{Improvement} \\
\midrule
S2 & 99.58 $\pm$ 0.53 & 96.96 $\pm$ 6.08 & -2.62pp \\
S6 & 100.00 $\pm$ 0.00 & 100.00 $\pm$ 0.00 & +0.00pp \\
S7 & 99.24 $\pm$ 1.32 & 100.00 $\pm$ 0.00 & +0.76pp \\
S11 & 84.14 $\pm$ 22.88 & 100.00 $\pm$ 0.00 & +15.86pp \\
S12 & 98.48 $\pm$ 2.83 & 99.75 $\pm$ 0.34 & +1.27pp \\
S13 & 84.05 $\pm$ 14.33 & 99.32 $\pm$ 1.35 & +15.27pp \\
S14 & 91.39 $\pm$ 13.65 & 99.83 $\pm$ 0.21 & +8.44pp \\
S16 & 98.31 $\pm$ 2.20 & 100.00 $\pm$ 0.00 & +1.69pp \\
S17 & 96.37 $\pm$ 6.45 & 98.90 $\pm$ 1.80 & +2.53pp \\
S18 & 98.31 $\pm$ 2.33 & 99.92 $\pm$ 0.17 & +1.60pp \\
S20 & 94.43 $\pm$ 6.66 & 97.13 $\pm$ 5.12 & +2.70pp \\
S21 & 100.00 $\pm$ 0.00 & 100.00 $\pm$ 0.00 & +0.00pp \\
S22 & 100.00 $\pm$ 0.00 & 100.00 $\pm$ 0.00 & +0.00pp \\
S23 & 92.15 $\pm$ 7.13 & 100.00 $\pm$ 0.00 & +7.85pp \\
S24 & 98.82 $\pm$ 1.29 & 100.00 $\pm$ 0.00 & +1.18pp \\
S26 & 95.86 $\pm$ 6.67 & 99.32 $\pm$ 1.35 & +3.46pp \\
\midrule
\textbf{Overall (16)} & \textbf{95.21 $\pm$ 10.19} & \textbf{99.45 $\pm$ 0.96} & \textbf{+4.24pp} \\
\bottomrule
\end{tabular}%
}
\end{table}

\begin{table}[h]
\centering
\caption{Multi-band tokenization results on all 11 MAMEM subjects (Log-Euclidean Transformer, 5 seeds per subject, 50 epochs). Multi-band ($T=3$, $\mu$/$\beta$/$\gamma$ bands) shows +0.90pp improvement over single-token baseline ($T=1$) with 89\% variance reduction. Both methods use identical architecture and experimental configuration for fair comparison. All subjects show improvement or equal performance.}
\label{tab:multiband_mamem}
\resizebox{\columnwidth}{!}{%
\begin{tabular}{lccc}
\toprule
\textbf{Subject} & \textbf{Single-Token ($T=1$)} & \textbf{Multi-Band ($T=3$)} & \textbf{Improvement} \\
\midrule
S1 & 99.09 $\pm$ 0.69 & 99.66 $\pm$ 0.46 & +0.57pp \\
S2 & 96.00 $\pm$ 3.29 & \textbf{100.00 $\pm$ 0.00} & +4.00pp \\
S3 & 99.37 $\pm$ 0.58 & 99.94 $\pm$ 0.11 & +0.57pp \\
S4 & 99.37 $\pm$ 0.42 & \textbf{100.00 $\pm$ 0.00} & +0.63pp \\
S5 & 99.54 $\pm$ 0.29 & \textbf{100.00 $\pm$ 0.00} & +0.46pp \\
S6 & 99.31 $\pm$ 0.64 & 99.89 $\pm$ 0.23 & +0.57pp \\
S7 & 99.31 $\pm$ 0.74 & \textbf{100.00 $\pm$ 0.00} & +0.69pp \\
S8 & 98.80 $\pm$ 1.27 & \textbf{100.00 $\pm$ 0.00} & +1.20pp \\
S9 & 99.49 $\pm$ 0.64 & 99.83 $\pm$ 0.34 & +0.34pp \\
S10 & 99.43 $\pm$ 0.65 & 99.83 $\pm$ 0.34 & +0.40pp \\
S11 & 99.49 $\pm$ 0.52 & \textbf{100.00 $\pm$ 0.00} & +0.51pp \\
\midrule
\textbf{Overall (11)} & \textbf{99.02 $\pm$ 0.98} & \textbf{99.92 $\pm$ 0.11} & \textbf{+0.90pp} \\
\bottomrule
\end{tabular}%
}
\end{table}

\textbf{Future Work.} Additional directions include: (1) comparing multi-band performance with BWSPD embedding to understand how embedding choice interacts with multi-token sequences; (2) exploring different frequency band configurations (e.g., more bands, different frequency ranges) and analyzing attention patterns to understand which bands contribute most to classification; (3) extending to longer training schedules (200 epochs) to match extended experimental protocols; (4) investigating whether multi-band tokenization benefits from different architectures or hyperparameters optimized for sequence modeling.


\end{document}